\documentclass{article}

\usepackage{microtype}
\usepackage{graphicx}
\usepackage{subfigure}
\usepackage{multirow}
\usepackage{booktabs} % for professional tables

\usepackage{hyperref}

\usepackage[accepted]{icml2025}

\usepackage{amsmath}
\usepackage{amssymb}
\usepackage{mathtools}
\usepackage{amsthm}
\usepackage{tabularx}
\newcommand*{\Scale}[2][4]{\scalebox{#1}{$#2$}}
\usepackage[capitalize,noabbrev]{cleveref}

\theoremstyle{plain}
\newtheorem{theorem}{Theorem}[section]
\newtheorem{proposition}[theorem]{Proposition}

\theoremstyle{definition}

\newtheorem{assumption}[theorem]{Assumption}
\theoremstyle{remark}

\usepackage[textsize=tiny]{todonotes}

\icmltitlerunning{Sequential Treatment Effect Estimation with Unmeasured Confounders}

\begin{document}

% Correcting Unmeasured Confounders in Causal Effects Estimation Under Sequential Treatments
\twocolumn[
\icmltitle{Sequential Treatment Effect Estimation with Unmeasured Confounders}



\icmlsetsymbol{equal}{*}

\begin{icmlauthorlist}
\icmlauthor{Yingrong Wang}{equal,sch1}
\icmlauthor{Anpeng Wu}{equal,sch1}
\icmlauthor{Baohong Li}{sch1}
\icmlauthor{Ziyang Xiao}{sch1}
\icmlauthor{Ruoxuan Xiong}{sch2}
\icmlauthor{Qing Han}{comp}
\icmlauthor{Kun Kuang}{sch1}
\end{icmlauthorlist}

\icmlaffiliation{sch1}{College of Computer Science and Technology, Zhejiang University, Hangzhou, China}
\icmlaffiliation{comp}{Ant Group, Hangzhou, China}
\icmlaffiliation{sch2}{Department of Quantitative Theory \& Methods, Emory University, Atlanta, USA.}

\icmlcorrespondingauthor{Kun Kuang}{kunkuang@zju.edu.cn}

\icmlkeywords{Machine Learning, ICML}

\vskip 0.3in
]



\printAffiliationsAndNotice{\icmlEqualContribution} % otherwise use the standard text.

\begin{abstract}
This paper studies the cumulative causal effects of sequential treatments in the presence of unmeasured confounders. It is a critical issue in sequential decision-making scenarios where treatment decisions and outcomes dynamically evolve over time. Advanced causal methods apply transformer as a backbone to model such time sequences, which shows superiority in capturing long time dependence and periodic patterns via attention mechanism. However, even they control the observed confounding, these estimators still suffer from unmeasured confounders, which influence both treatment assignments and outcomes. How to adjust the latent confounding bias in sequential treatment effect estimation remains an open challenge. Therefore, we propose a novel \textbf{D}ecomposing \textbf{S}equential \textbf{I}nstrumental \textbf{V}ariable framework for \textbf{C}ounter\textbf{F}actual \textbf{R}egression (\textbf{DSIV-CFR}), relying on a common negative control assumption. Specifically, an instrumental variable (IV) is a special negative control exposure, while the previous outcome serves as a negative control outcome. This allows us to recover the IVs latent in observation variables and estimate sequential treatment effects via a generalized moment condition. We conducted experiments on $4$ datasets and achieved significant performance in one- and multi-step prediction, supported by which we can identify optimal treatments for dynamic systems. 
\end{abstract}

\section{Introduction}

Sequential decision-making is fundamental to many real-world applications, including personalized medicine~\cite{medicine1, medicine2}, financial investment~\cite{finance1,finance2}, and policy making~\cite{policy1,policy2}. These scenarios involve decisions that must dynamically adapt to evolving conditions, where the outcomes of previous decisions directly influence subsequent choices. For example, we consider a cancer patient undergoing treatment, as illustrated in Figure~\ref{fig:example}. The medical team must regularly adjust the treatment plan based on the patient’s condition (confounders) to control tumor volume (outcome)~\cite{dataset-tumor}. After surgery, if the tumor grows to $55\ \text{cm}^3$, the team may choose chemotherapy for the next stage after comprehensively considering the patient’s current health and expected treatment effects. Similar settings arise in various domains, such as financial investment, where decisions must adapt to fluctuating market conditions, and supply chain management, where strategies are adjusted dynamically based on demand and logistical feedback. Understanding the causal effects of sequential treatments is crucial for optimizing decision-making in such scenarios. However, accurately estimating these effects is complicated by the presence of unmeasured confounders that simultaneously influence treatment assignments and outcomes but remain unobserved. Unmeasured confounding introduces significant bias into causal estimates, leading to unreliable conclusions and potentially suboptimal or harmful decisions~\cite{robins1986role,kuroki2014measurement}.

\begin{figure}[t]
\begin{center}
\centerline{\includegraphics[width=0.5\textwidth]{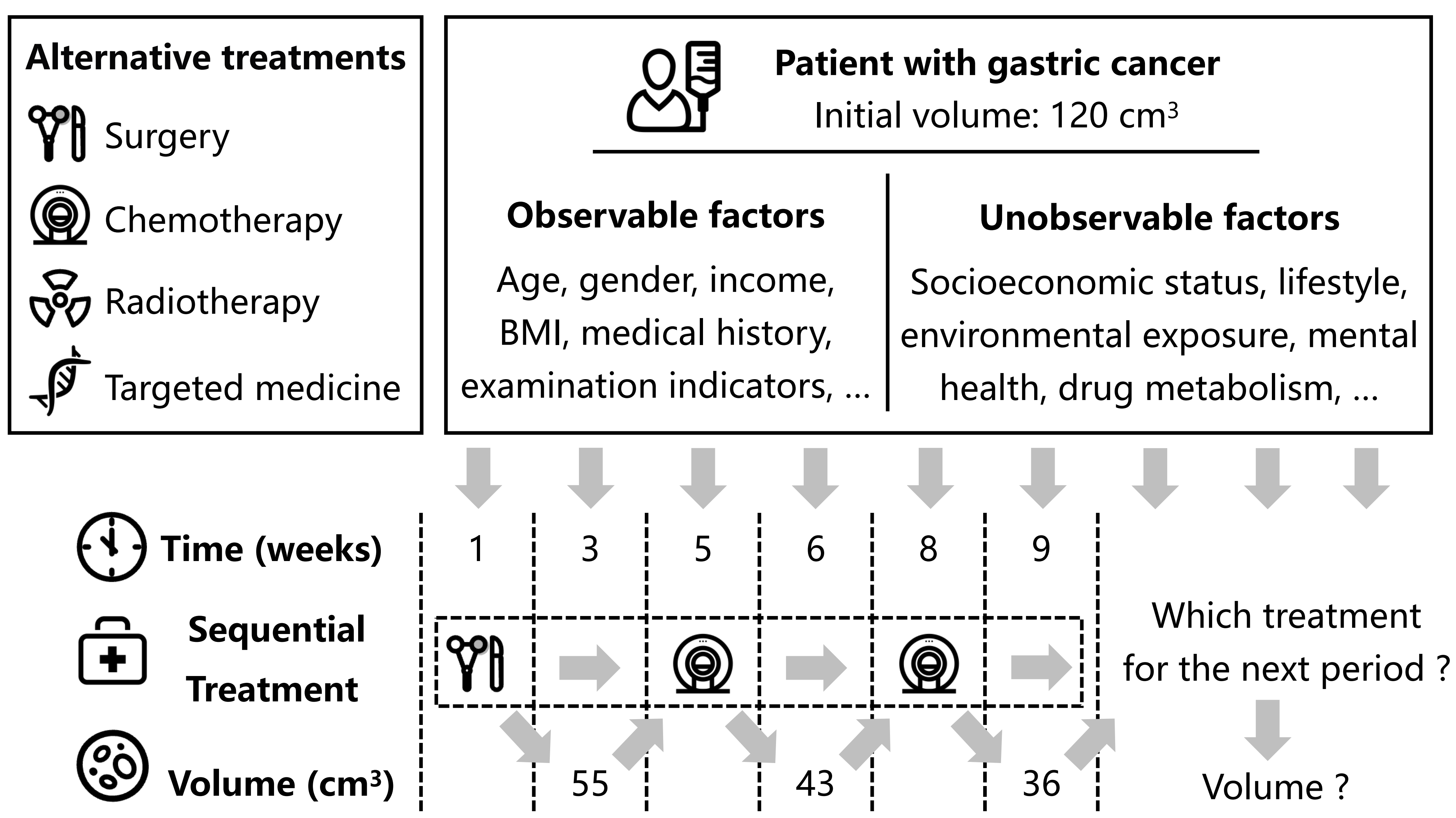}}
\vskip -0.1in
\caption{A case of counterfactual prediction and decision making on the time series data in a medical setting.}
\label{fig:example}
\end{center}
\vskip -0.4in
\end{figure}

Several studies have explored causal effect estimation from observational time series data. One significant challenge in these settings is the time dependency of variables, where all variables are influenced not only by causal relationships but also by their prior values over time. Early works leveraged the Markov property to simplify this issue, modeling the problem as a state transition chain where outcomes depend only on the immediate past~\cite{keith2021estimating,ORL}. More broadly, autoregressive models like RNNs~\cite{RNN} and LSTMs~\cite{LSTM} incorporate all historical data to estimate future outcomes, with approaches such as ACTIN~\cite{ACTIN} adapting LSTMs for causal inference in time sequences. However, these methods often face computational challenges when handling high-dimensional, long time series data. Recently, transformers~\cite{transformer} have emerged as a more powerful tool for identifying long-time dependence and periodic patterns due to its attention mechanism, which dynamically captures the relationships between different positions in a sequence and flexibly allocates their weights. However, these approaches~\cite{CT,DLTMLE} rely on the unconfoundedness assumption, which presumes that all confounders are observed. In practice, this assumption is rarely met due to the difficulty of measuring critical latent factors, leading to biased causal effect estimates and unreliable findings.

Unmeasured confounding remains a critical challenge in causal inference, particularly in sequential treatment settings. For example, as illustrated in Figure~\ref{fig:example}, in cancer care, factors like socioeconomic status, mental health, and unmonitored lifestyle habits significantly influence treatment decisions and outcomes, yet these factors are often difficult or impossible to quantify or measure. Ignoring these variables would lead to biased estimates of causal effects, distorting the true relationships between treatments and outcomes. While existing approaches, such as Time Series Deconfounder~\cite{TSD}, have made strides in addressing unmeasured confounding, they are often constrained by assumptions that may not hold in real-world scenarios, such as restrictive data requirements or limited flexibility in model design. Additionally, many of these methods face challenges with scalability and robustness, particularly when applied to high-dimensional, long time series data where the complexity of relationships between variables increases exponentially.

To address these challenges, we propose a novel framework: the Decomposing Sequential Instrumental Variable Framework for Counterfactual Regression (DSIV-CFR). This framework builds on the common negative control assumption, treating instrumental variables (IVs) as special negative control exposures and prior outcomes as negative control outcomes. These relationships allow the recovery of instrumental variables from observed covariates, enabling the framework to mitigate bias introduced by unmeasured confounders. DSIV-CFR effectively decomposes the problem into manageable tasks using generalized moment conditions, making it robust for handling high-dimensional, time-dependent data. Unlike existing methods, DSIV-CFR does not rely on the unconfoundedness assumption, ensuring reliable causal effect estimates even in the presence of latent confounders. Through extensive experiments on synthetic and real-world datasets, we demonstrate that DSIV-CFR significantly outperforms existing methods, providing a scalable, accurate, and practical solution for optimizing sequential decision-making in dynamic systems.

\begin{figure*}[t]
% \vskip 0.1in
\begin{center}
\centerline{\includegraphics[width=0.85\textwidth]{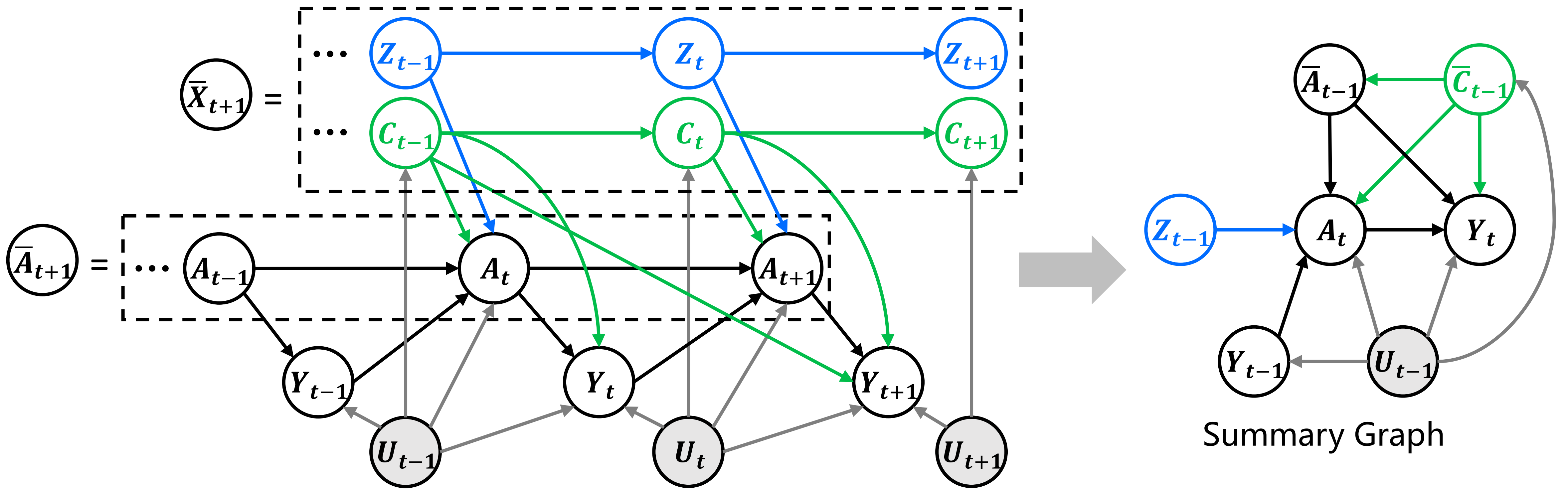}}
\vskip -0.1in
\caption{Casual graphs for illustrating the relationships between different variables from time series data. On the left, we describe the causalities in details, where the subscript associated with $t$ in each variable reflects the time dependence. Unobserved variables are marked in shadow. In this paper, we focus on the cumulative effect of treatment $(\boldsymbol{A}_1,\dots,\boldsymbol{A}_t)$ on outcome $\boldsymbol{Y}_t$. For simplicity, we give its summary graph on the right, where the historical information of all treatments and confounders are denoted as $\bar{\boldsymbol{A}}_{t-1}$ and $\bar{\boldsymbol{C}}_{t-1}$.}
\label{fig:causal-graph}
\end{center}
\vskip -0.3in
\end{figure*}

Contributions of this paper can be concluded as follows:
\vspace{-10pt}
\begin{itemize}
    \item We systemically study the problem of sequential treatment effect estimation, particularly the complex causal relationships among variables in time series data under unmeasured confounders.
    \vspace{-3pt}
    \item We proposed a novel method, DSIV-CFR, which effectively addresses the bias caused by unobserved confounders by leveraging instrumental variables and prior outcomes as negative controls, enabling accurate estimation of sequential treatment effects.
    \vspace{-3pt}
    \item Extensive experiments demonstrate the superiority of DSIV-CFR in mitigating unmeasured confounding bias, improving causal estimation accuracy, and identifying optimal treatment strategies in dynamic systems.
\end{itemize}

\section{Related Work}\label{sec:related}
\textbf{Treatment effect estimation for time series data.} There is a comprehensive survey~\cite{time-series-survey} that explores some basic problems about causal inference for time series analysis. Variations of treatments in these scenarios include time-invariant and time-varying ones. Time-invariant treatment refers to an intervention that occurs at a specific time and remains unchanged afterward. To estimate the effect of such treatment, the difference-in-difference method~\cite{did-orley,did-susan} is one of the mostly used tools in economics, which depend on parallel trends assumption. Estimation of effects for time-varying treatment is most relevant to this paper. One solution to model the time series data is to utilize the Markov assumption, where the complex temporal issue can be simplified as a problem of state transitions between adjacent timestamps~\cite{keith2021estimating,ORL}, i.e. the data at time $t$ is influenced only by the data at time $t-1$. The other way is to apply autoregressive models to connect causal models at a single timestamp, where all historical information serves as input to affect the current time point $t$. For example, Time Series Deconfounder~\cite{TSD} uses RNN~\cite{RNN} can be regarded as an extension of Deconfounder~\cite{deconfounder} to address the temporal sequences. LSTM~\cite{LSTM} is also applied to extend the ideas of CFR~\cite{CFR} to the estimation of causal effects in time series data~\cite{ACTIN}. In recent works~\cite{CT,DLTMLE}, transformer~\cite{transformer} serves as a more powerful tool to model complex time series when estimating the treatment effect.

\textbf{Treatment effect estimation with unobserved confounders.} Instrumental variable (IV) is a powerful tool to address latent confounding bias. It is independent of unmeasured confounders, serves as a cause of treatment, but has no influence on the outcome. However, standard IV methods require a predefined, strong, and valid IV, which is often difficult to find in real-world applications. Therefore, many researchers have studied the generation or synthesis of IV to address this problem. For example, ModelIV~\cite{modelIV} identifies valid IVs by selecting those in the tightest cluster of estimation points, assuming that these candidates yield approximately causal effects, thus relaxing the requirement for more than half of the candidates to be valid. AutoIV~\cite{autoIV} generates IV representations by learning a disentangled representation from the observational data on the basis of independence conditions. However, sometimes the representations learned by this methodology may be a reverse IV rather than a true IV. That is, there is a causal link from the treatment variable that leads to the so-called IV. Negative control is another tool that is widely used to address the issue of unmeasured confounding. Negative control refers to observational variables that could serve as proxies to indirectly indicate the effects of unmeasured confounders~\cite{miao2018identifying}. A cross-moment approach~\cite{yaroslav2023cross} is introduced, which also leverages the idea of negative controls to mitigate the impact of unobserved confounders. Moreover, negative controls can be naturally extended to time series data analysis~\cite{miao2024bridge}. In this context, the treatment applied at the next moment can be regarded as a negative control exposure, while the outcome at the previous timestamp serves as a negative control outcome.

\section{Problem Setup}\label{sec:formulation}

In this paper, we aim to study the cumulative treatment effects of sequential treatments under unmeasured confounders. For observational data, as shown in Figure~\ref{fig:causal-graph}, we can observe $\mathcal{D} = \{\boldsymbol{A}_{t},\boldsymbol{X}_{t},\boldsymbol{Y}_{t}\}_{t=1}^{T}$ collected over $T$ time steps. Among the observational variables, $\boldsymbol{A}_{t}=\{\boldsymbol{A}_{t}^i\}_{i=1}^n$ represents the treatments applied at timestamp $t$ on $n$ units. Generally in this paper, the superscript $i$ denotes the index of each unit, while the subscript $t$ denotes the timestamp. $\boldsymbol{X}_{t}$ and $\boldsymbol{Y}_{t}$ are the observed covariates and factual outcomes. In the presence of unmeasured confounders, there exists a set of latent variables $\boldsymbol{U}_{t}$, which are missing and may simultaneously affect both the treatments $\boldsymbol{A}_t$ and the outcomes $\boldsymbol{Y}_{t-1}$ and $\boldsymbol{Y}_t$. These unobserved confounders introduce dependencies that confuse causal inference. In this paper, we use $\bar{\boldsymbol{A}}_{t-1}=(\boldsymbol{A}_1,\dots,\boldsymbol{A}_{t-1})$ to denote the history of the treatment variable $\boldsymbol{A}$ up to time $t-1$. Similarly, we define $\bar{\boldsymbol{X}}_{t-1}$,$\bar{\boldsymbol{Y}}_{t-1}$, and $\bar{\boldsymbol{U}}_{t-1}$. For simplicity, we use $\bar{\boldsymbol{H}}_{t-1}=\{\bar{\boldsymbol{X}}_{t-1},\bar{\boldsymbol{A}}_{t-1},\bar{\boldsymbol{Y}}_{t-1}\}$ to denote observed history before treatment $\boldsymbol{A}_t$ is applied. In the sequential treatment effect problem, the sequential value on $\{\bar{\boldsymbol{A}}_{t-1}, \bar{\boldsymbol{X}}_{t-1}, \bar{\boldsymbol{U}}_{t-1}\}$ act as confounders that jointly influence both the current treatment $\boldsymbol{A}_t$ and the outcome $\boldsymbol{Y}_t$.

One classical solution to reduce bias from $\boldsymbol{U}_{t-1}$ is to use a predefined IV, leveraging its exogeneity to infer causal effects. However, it is hard to find such IVs in reality. Therefore, we aim to recover IVs from observations $\boldsymbol{X}_{t-1}$ by decomposing them into the following two parts\footnote{In time series analysis, $\boldsymbol{Z}$ and $\boldsymbol{C}$ are considered two independent sources of $\boldsymbol{X}$, with the mapping relationships assumed to be static, meaning they remain unchanged over time.}: (1) confounders $\boldsymbol{C}_{t-1}$ that $\boldsymbol{A}_{t}\leftarrow\boldsymbol{C}_{t-1}\rightarrow\boldsymbol{Y}_{t}$. (2) instruments $\boldsymbol{Z}_{t-1}$ that affect $\boldsymbol{A}_t$ but are not influenced by $\boldsymbol{U}_{t-1}$. Under the proposed time-series framework, we reconstruct them as a summary graph, which is shown in the right of Figure~\ref{fig:causal-graph}. The components of this summary graph include newly applied treatment $\boldsymbol{A}_t$, outcome of interest $\boldsymbol{Y}_t$, instruments $\boldsymbol{Z}_{t-1}$, confounders $\{\bar{\boldsymbol{A}}_{t-1},\bar{\boldsymbol{C}}_{t-1},\boldsymbol{U}_{t-1}\}$, and auxiliary information $\boldsymbol{Y}_{t-1}$. Our task is to predict the counterfactual outcome for the choice of optimal treatment given $\bar{\boldsymbol{H}}_{t-1}$:
\vspace{-2pt}
\begin{equation}
    \mathbb{E}[\boldsymbol{Y}_{t}(\boldsymbol{a}_{t})|\bar{\boldsymbol{H}}_{t-1}].
\end{equation}
\vspace{-22pt}

We make the following assumptions throughout this paper.

% \begin{definition}
% \textbf{(Instrumental Variable).} An IV $\boldsymbol{Z}_{t-1}$ satisfies the exclusion requirement that $\boldsymbol{Z}_t$ does not directly affect $\boldsymbol{Y}_{t+1}$, i.e. $\boldsymbol{Z}_{t-1}\perp\!\!\!\perp\boldsymbol{Y}_t\ |\ \boldsymbol{A}_t,\boldsymbol{C}_{t-1},\boldsymbol{Y}_{t-1}$.
% \label{def:IV}
% \end{definition}

\begin{assumption}
\textbf{(Consistency).} If a unit receives a treatment, i.e. $\boldsymbol{A}_{t}=\boldsymbol{a}_{t}$, then the observed
outcomes $\boldsymbol{Y}_{t}$ are identical to the potential outcomes $\boldsymbol{Y}_{t}(\boldsymbol{a}_{t})$.
\label{ass:consistency}
\end{assumption}

\vspace{-2pt}
\begin{assumption}
\textbf{(Overlap).} For any intervention, there is always a positive probability of receiving it given confounders, i.e.
$\sup_{\boldsymbol{a}_t\in\boldsymbol{A}_t}|\frac{\mathbb{P}(\boldsymbol{A}_t=\boldsymbol{a}_t|\bar{\boldsymbol{C}}_{t-1})}{\mathbb{P}(\boldsymbol{A}_t=\boldsymbol{a}_t|\bar{\boldsymbol{C}}_{t-1},\boldsymbol{U}_{t-1})}|<\infty$.
\label{ass:overlap}
\end{assumption}

\vspace{-2pt}
\begin{assumption}
\textbf{(Sequential Latent Ignorability).} By controlling all confounders, including observed ones $\boldsymbol{C}_{t-1}$ and unobserved ones $\boldsymbol{U}_{t-1}$, the effect of treatment $\boldsymbol{A}_{t}$ on the outcome $\boldsymbol{Y}_{t}$ could be estimated without bias. Formally, $\boldsymbol{Y}_{t}(\boldsymbol{a}_{t}) \perp\!\!\!\perp \boldsymbol{A}_{t}\ |\ \{\bar{\boldsymbol{C}}_{t-1},\boldsymbol{U}_{t-1}\}$ for all $\boldsymbol{a}_{t}$.
\label{ass:latent-ignore}
\end{assumption}

\vspace{-2pt}
\begin{assumption}
\textbf{(Additive Noise Model).} Error terms are independent of each other, that is, the unobserved confounders from the previous time point do not affect the next ones. Formally, $\boldsymbol{Y}_{t-1}\not\rightarrow\boldsymbol{Y}_t$ and $\boldsymbol{U}_{t-1}\not\rightarrow\boldsymbol{U}_t$. To estimate the potential outcome, we also assume $\mathbb{E}[\boldsymbol{U}|\boldsymbol{X}]=0$.
\label{ass:additive}
\end{assumption}

\vspace{-2pt}
\begin{assumption}
\textbf{(Time-invariant Relationship).} The function of the treatment effect, denoted as $h:\boldsymbol{X}\times\boldsymbol{A}\rightarrow\boldsymbol{Y}$, does not change over time, although its value fluctuates due to the confounding heterogeneity of $\boldsymbol{X}$.
\label{ass:time-invariant}
\end{assumption}

\begin{figure*}[t]
% \vskip 0.2in
\begin{center}
\centerline{\includegraphics[width=1\textwidth]{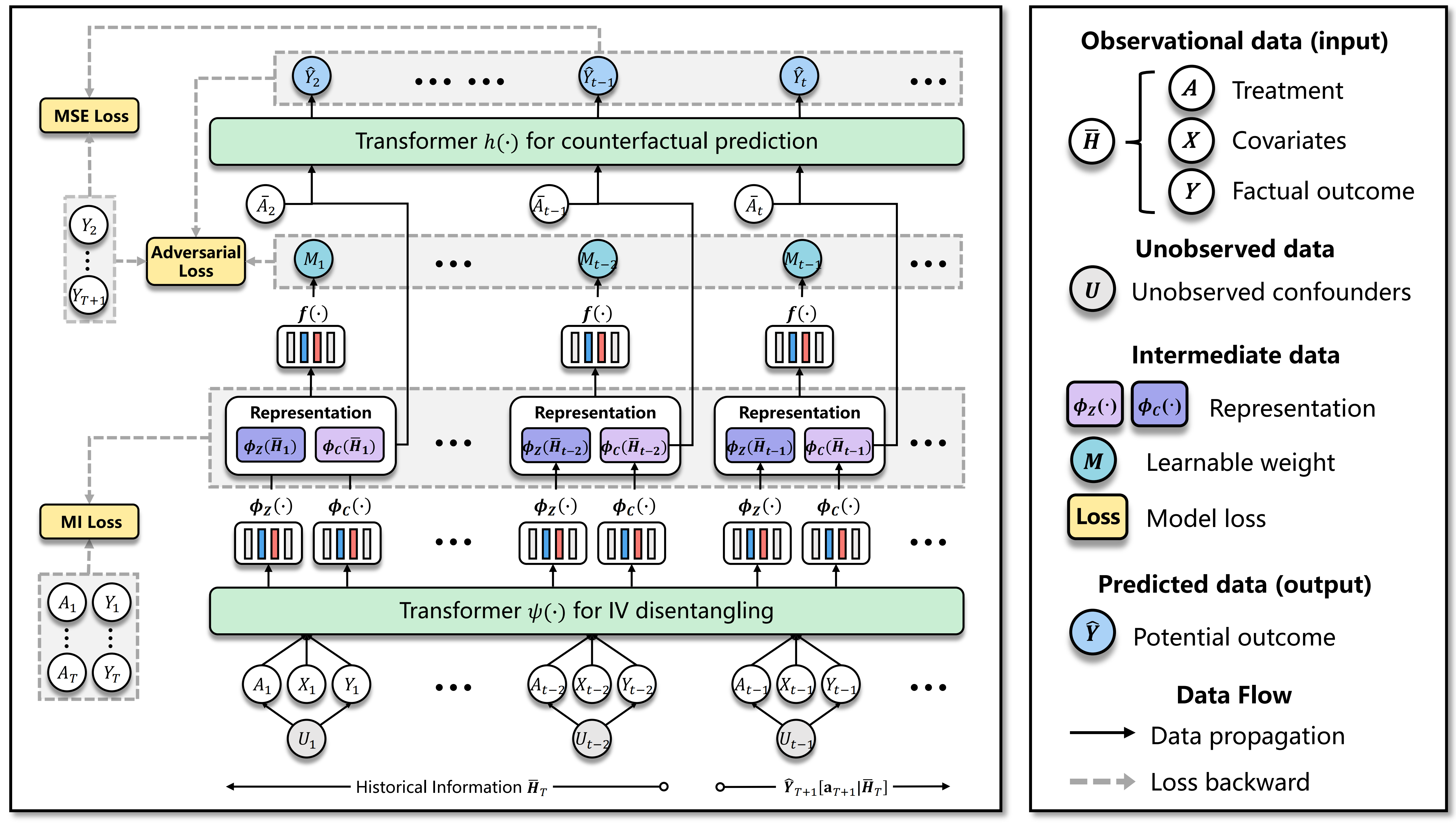}}
\vskip -0.05in
\caption{Overview of our model DSIV-CFR. Some explanations have been displayed in the right panel. Historical observations $\bar{\boldsymbol{H}}$ are input into the first transformer $\psi(\cdot)$ to learn the representation of IVs $\phi_Z$ and confounders $\phi_C$, which is optimized by the mutual information (MI) loss. The second transformer $h(\cdot)$ is trained as a backbone with the objective of accurately predicting future potential outcomes $\boldsymbol{Y}$, measured by the MSE loss. In addition, an adversarial loss derived from IVs and confounders is considered to build up a GMM framework, requiring another bridge function $f(\cdot)$ to learn a weight $\boldsymbol{M}$ so as to train against $h(\cdot)$.}
\label{fig:model}
\end{center}
\vskip -0.3in
\end{figure*}

\vspace{-2pt}
\begin{proposition}[\textbf{IV Decomposition}]\label{pro:decompose}
Following the IV conditions, if all variables are observable, we can decompose the instrument ${\boldsymbol{Z}}_{t-1}$ from the observed covariates $\bar{\boldsymbol{X}}_{t-1}$ as follows: ${\boldsymbol{Z}}_{t-1} \perp\!\!\!\perp \bar{\boldsymbol{C}}_{t-1}$ and ${\boldsymbol{Z}}_{t-1} \perp\!\!\!\perp \boldsymbol{Y}_t \mid \{\boldsymbol{A}_t, \bar{\boldsymbol{A}}_{t-1}, \bar{\boldsymbol{C}}_{t-1}, {\boldsymbol{U}}_{t-1}\}$.
\end{proposition}
\vspace{-6pt}

These conditions imply that the instrument variable ${\boldsymbol{Z}}_{t-1}$ is independent of the outcome $\boldsymbol{Y}_t$ given the appropriate set of covariates, including the unmeasured confounders $\boldsymbol{U}_{t-1}$. However, in practice, we cannot directly identify the instrument $\boldsymbol{Z}_{t-1}$ due to $\boldsymbol{U}_{t-1}$. As a result, instrumental variable identification becomes challenging and would require additional assumptions. Motivated by the idea of negative control~\cite{miao2018identifying,miao2024bridge}, we utilize a pair of proxy variables, including negative control exposure and negative control outcome, to help identify causal effects.

% \vspace{-2pt}
\begin{assumption}
\textbf{(Negative Control).} 
Negative controls comprise negative control exposures (NCE) and negative control outcomes (NCO).
NCE cannot directly affect the outcome $\boldsymbol{Y}_t$, and neither NCE nor treatment $\boldsymbol{A}_t$ can affect NCO. The effect of unmeasured confounders $\boldsymbol{U}_{t-1}$ on NCE and $\boldsymbol{A}_t$, as well as on NCO and $\boldsymbol{Y}_t$, is proportional.
\label{ass:negative}
\end{assumption}
\vspace{-3pt}
Generally, we can treat the instrument $\boldsymbol{Z}_{t-1}$ as a special case of NCE, while previous outcome $\boldsymbol{Y}_{t-1}$ can also be regarded as a special case of NCO.

\begin{theorem}[\textbf{IV Identification}]\label{theorem:IV}
    Based on the negative control assumptions, given $\bar{\boldsymbol{C}}_{t-1}$, we can decompose the instrument ${\boldsymbol{Z}}_{t-1}$ from the observed covariates $\bar{\boldsymbol{X}}_{t-1}$ as follows: ${\boldsymbol{Z}}_{t-1} \perp\!\!\!\perp \bar{\boldsymbol{C}}_{t-1}$ and ${\boldsymbol{Z}}_{t-1} \perp\!\!\!\perp \boldsymbol{Y}_t \mid \{\boldsymbol{A}_t, \bar{\boldsymbol{A}}_{t-1}, \bar{\boldsymbol{C}}_{t-1}, {\boldsymbol{Y}}_{t-1}\}$.
\end{theorem}

\begin{proof}
    Under Assumptions~\ref{ass:consistency}-\ref{ass:time-invariant} and \ref{ass:negative}, we assume that the function of $\boldsymbol{U}_{t-1}$ on $\boldsymbol{Y}_{t-1}$ and $\boldsymbol{Y}_{t}$ is the same, and the outcome feedback function $h$ can be reformulated as:
\begin{equation}\label{eq:outcome}
\begin{split}
    \boldsymbol{Y}_t & = h(\boldsymbol{A}_t, \bar{\boldsymbol{A}}_{t-1}, \bar{\boldsymbol{C}}_{t-1}) + \epsilon(\boldsymbol{U}_{t}) + \epsilon(\boldsymbol{U}_{t-1}) 
    \\
    \boldsymbol{Y}_{t-1} & = \Scale[0.9]{h(\boldsymbol{A}_{t-1}, \bar{\boldsymbol{A}}_{t-2}, \bar{\boldsymbol{C}}_{t-2}) + \epsilon(\boldsymbol{U}_{t-2}) + \epsilon(\boldsymbol{U}_{t-1})},
\end{split}
\end{equation}
where $\epsilon(\boldsymbol{U}_{t-1})$ denotes the unmeasured common causes of treatment $\boldsymbol{A}_t$ and outcome $\boldsymbol{Y}_t$. Given $\boldsymbol{Z}_{t-1} \perp\!\!\!\perp \epsilon(\boldsymbol{U}_{t}) \mid \{\boldsymbol{A}_t, \bar{\boldsymbol{A}}_{t-1}, \bar{\boldsymbol{C}}_{t-1}\}$ and $\boldsymbol{Z}_{t-1} \not \! \perp \!\!\! \perp \epsilon(\boldsymbol{U}_{t-1}) \mid \{\boldsymbol{A}_t, \bar{\boldsymbol{A}}_{t-1}, \bar{\boldsymbol{C}}_{t-1}\}$, we can reformulate the decomposition condition in Proposition~\ref{pro:decompose} as follows:
\begin{equation}
\begin{split}
    {\boldsymbol{Z}}_{t-1} \perp\!\!\!\perp &\{h(\boldsymbol{A}_t, \bar{\boldsymbol{A}}_{t-1}, \bar{\boldsymbol{C}}_{t-1}) + \epsilon(\boldsymbol{U}_{t-1})\}\\ 
    \ \mid\  &\{\boldsymbol{A}_t, \bar{\boldsymbol{A}}_{t-1}, \bar{\boldsymbol{C}}_{t-1}, \epsilon(\boldsymbol{U}_{t-1})\}.
\end{split}
\end{equation}
Therefore, we can identify IV ${\boldsymbol{Z}}_{t-1}$ from:
\begin{equation}
\begin{split}
    {\boldsymbol{Z}}_{t-1} \perp\!\!\!\perp &\{h(\boldsymbol{A}_t, \bar{\boldsymbol{A}}_{t-1}, \bar{\boldsymbol{C}}_{t-1}) + \epsilon(\boldsymbol{U}_{t-1})\}  \\
    \ \mid\ &\{\boldsymbol{A}_t, \bar{\boldsymbol{A}}_{t-1}, \bar{\boldsymbol{C}}_{t-1}, \boldsymbol{Y}_{t-1}\}.
\end{split}
\end{equation}
Then, we can use the instrument ${\boldsymbol{Z}}_{t-1}$ to help us isolate the direct causal effect $h(\boldsymbol{A}_t, \bar{\boldsymbol{A}}_{t-1}, \bar{\boldsymbol{C}}_{t-1})$ from unmeasured confounding bias.
\end{proof}

By taking the expectation of both sides of Equation (\ref{eq:outcome}) conditioned on $\{\bar{\boldsymbol{C}}_{t-1},\boldsymbol{Z}_{t-1}\}$, given Assumption \ref{ass:additive} $\mathbb{E}[\boldsymbol{U}_{t}|\bar{\boldsymbol{X}}_{t-1}]=0$, we have:
\begin{equation}\label{eq:2sls}
\begin{split}
    &\mathbb{E}[\boldsymbol{Y}_t
    |\bar{\boldsymbol{C}}_{t-1},\boldsymbol{Z}_{t-1}] \\
    =&\Scale[0.9]{\mathbb{E}[h(\boldsymbol{A}_t,\bar{\boldsymbol{A}}_{t-1},\bar{\boldsymbol{C}}_{t-1})|\bar{\boldsymbol{C}}_{t-1},\boldsymbol{Z}_{t-1}] + \mathbb{E}[\epsilon(\boldsymbol{U})_{t-1}|\bar{\boldsymbol{C}}_{t-1}]} \\
    =&\Scale[0.9]{\int h(\boldsymbol{A}_{t},\bar{\boldsymbol{A}}_{t-1},\bar{\boldsymbol{C}}_{t-1})\ dF(\boldsymbol{A}_{t}|\bar{\boldsymbol{A}}_{t-1},\bar{\boldsymbol{C}}_{t-1},\boldsymbol{Z}_{t-1})},
\end{split}
\end{equation}
where $F(\boldsymbol{A}_{t}|\bar{\boldsymbol{A}}_{t-1},\bar{\boldsymbol{C}}_{t-1},\boldsymbol{Z}_{t-1})$ is the conditional treatment distribution. Equation (\ref{eq:2sls}) expresses an inverse problem for the function $h$ in terms of $\mathbb{E}[\boldsymbol{Y}_t|\bar{\boldsymbol{C}}_{t-1},\boldsymbol{Z}_{t-1}]$ and $F(\boldsymbol{A}_{t}|\bar{\boldsymbol{A}}_{t-1},\bar{\boldsymbol{C}}_{t-1},\boldsymbol{Z}_{t-1})$~\cite{newey2003instrumental}. A standard approach is two-stage regression: using observed confounders and instrument to predict conditional treatment distribution $\hat{F}(\boldsymbol{A}_{t}|\bar{\boldsymbol{A}}_{t-1},\bar{\boldsymbol{C}}_{t-1},\boldsymbol{Z}_{t-1})$, and then regressing outcome with estimated treatment $\hat{\boldsymbol{A}}_t$ sampled from $\hat{F}$ and observed confounders, i.e $\hat{h}=(\hat{\boldsymbol{A}}_t,\bar{\boldsymbol{A}}_{t-1}, \bar{\boldsymbol{C}}_{t-1})$. It can be regarded as a special form of the generalized method of moments~\cite{GMM}, abbreviated as GMM. In the context of negative controls~\cite{miao2024bridge}, we would like to extend this consistent estimation by applying GMM.

\section{Methodology}\label{sec:methodology}

Our proposed method, the Decomposing Sequential Instrumental Variable framework for Counterfactual Regression (DSIV-CFR), is illustrated in Figure~\ref{fig:model}. It consists of two key modules designed to address the challenges of sequential treatment effect estimation with unmeasured confounders. The first module focuses on leveraging transformers to model long sequential dependencies and learning the representations of instrumental variables (IVs) and confounders. The second module employs the generalized method of moments (GMM) to estimate counterfactual outcomes while effectively utilizing the learned representations.

\subsection{Long Sequential Modeling for Learning IV and Confounder Representations}\label{sec:IV}

In our setting for estimating the effect of sequential treatment, the relationships between variables are influenced not only by immediate past values but also by long-term dependencies. To capture these complex dependencies, we employ transformer~\cite{transformer,CT} as our backbone, which has shown state-of-the-art performance in modeling long-range interactions in sequential data. Specifically, we build a $6$-layer transformer with $8$ heads. The input to this module consists of the observed covariates, treatment variables, and outcomes at each time step, i.e., $\bar{\boldsymbol{H}}_{t-1}=\{\bar{\boldsymbol{X}}_{t-1},\bar{\boldsymbol{A}}_{t-1},\bar{\boldsymbol{Y}}_{t-1}\}$. In each head $j$, we calculate the keys $K^{(j)}$, queries $Q^{(j)}$, and values $V^{(j)}$ from a linear transformation of $\bar{\boldsymbol{H}}_{t-1}$ to obtain:
\vspace{-2pt}
\begin{equation}
\text{Atten}^{(j)}=\text{softmax}\left(\frac{Q^{(j)}K^{(j)\top}}{\sqrt{d_{qkv}}}\right)V^{(j)},
\end{equation}
where $d_{qkv}$ is the dimensionality. We use the concatenation as the output of this sub-module with $J$ heads:
% \vspace{-11pt}
\begin{equation}
\text{MHA}=\text{concat}\left(\text{Atten}^{(1)},\dots,\text{Atten}^{(J)}\right).
\end{equation}
The feed-forward layer with ReLU activation is then used for the non-linear transformations of MHA. Techniques such as layer normalization~\cite{layer-norm} and dropout~\cite{dropout} are also applied to enhance the stability and robustness of model training. The output of this module is a basic embedding $\psi(\bar{\boldsymbol{H}}_{t-1})$, which is then followed by $\phi_Z(\cdot)$ and $\phi_C(\cdot)$ to further disentangle IVs and confounders.

\textbf{Learning IV Representation.} Two conditions the IV should be satisfied: (1) Relevance. Instruments $\boldsymbol{Z}_{t-1}$ are required to be correlated with treatment $\boldsymbol{A}_{t}$. Inspired by AutoIV~\cite{autoIV}, we encourage the information of $\boldsymbol{Z}_{t-1}$ related to $\boldsymbol{A}_{t}$ to enter the IV representations $\phi_Z(\psi(\bar{\boldsymbol{H}}_{t-1}))$ by minimizing the additive inverse of contrastive log-ratio upper bound~\cite{CLUB}:
\begin{equation}\label{eq:mi-za}
\begin{split}
    \mathcal{L}_{ZA,t}=&-\frac{1}{n^2}\sum_{i=1}^n\sum_{j=1}^n[\ \log q_{\theta_{ZA}}(\boldsymbol{A}^i_{t}|\phi_Z(\psi(\bar{\boldsymbol{H}}^i_{t-1}))) \\
    &-\log q_{\theta_{ZA}}(\boldsymbol{A}^j_{t}|\phi_Z(\psi(\bar{\boldsymbol{H}}^i_{t-1})))\ ].
\end{split}
\end{equation}
The variational distribution $q_{\theta_{ZA}}(\boldsymbol{A}_{t}|\phi_Z(\psi(\bar{\boldsymbol{H}}_{t-1})))$ is determined by parameters $\theta_{ZA}$ to approximate the true conditional distribution $\mathbb{P}(\boldsymbol{A}_{t}|\phi_Z(\psi(\bar{\boldsymbol{H}}_{t-1})))$. Log-likelihood loss function of variational approximation is defined as:
\begin{equation}\label{eq:lld-za}
\mathcal{LLD}_{ZA,t}=-\frac{1}{n}\sum_{i=1}^n\log q_{\theta_{ZA}}(\boldsymbol{A}^i_{t}|\phi_Z(\psi(\bar{\boldsymbol{H}}^i_{t-1}))).
\end{equation}
(2) Exclusion. As mentioned in Theorem~\ref{theorem:IV}, we require ${\boldsymbol{Z}}_{t-1} \perp\!\!\!\perp \boldsymbol{Y}_t \mid \{\boldsymbol{A}_t, \bar{\boldsymbol{A}}_{t-1}, \bar{\boldsymbol{C}}_{t-1}, {\boldsymbol{Y}}_{t-1}\}$. Similarly,
\begin{equation}\label{eq:mi-zy}
\begin{split}
    \mathcal{L}_{ZY,t}=\frac{1}{n^2}\sum_{i=1}^n\sum_{j=1}^n\{&w^{ij}_{t-1}[\ \log q_{\theta_{ZY}}(\boldsymbol{Y}^i_{t}|\phi_Z(\psi(\bar{\boldsymbol{H}}^i_{t-1}))) \\
    &-\log q_{\theta_{ZY}}(\boldsymbol{Y}^j_{t}|\phi_Z(\psi(\bar{\boldsymbol{H}}^i_{t-1})))\ ]\ \},
\end{split}
\end{equation}
where $w_{t-1}^{ij}$ is the weight of sample pair $(i,j)$ to achieve the conditional independence mentioned above. It is determined by the discrepancy between $v_t^i=[\boldsymbol{A}_t^i, \bar{\boldsymbol{A}}_{t-1}^i, \bar{\boldsymbol{C}}_{t-1}^i, {\boldsymbol{Y}}_{t-1}^i]$ and $v_t^j=[\boldsymbol{A}_t^j, \bar{\boldsymbol{A}}_{t-1}^j, \bar{\boldsymbol{C}}_{t-1}^j, {\boldsymbol{Y}}_{t-1}^j]$ in the RBF kernel:
\begin{equation}
    w_{t-1}^{ij}=\text{softmax}\left(\exp\left(-\frac{\lVert v_i-v_j\lVert^2}{2\sigma^2}\right)\right).
\end{equation}
The hyper-parameter $\sigma$ is used to control the width of the Gaussian function, which we set to $1$ in the experiments.

\textbf{Learning Confounder Representation.} There are also two restrictions that should be taken into account: (1) $\boldsymbol{A}_{t}\leftarrow\boldsymbol{C}_{t-1}\rightarrow\boldsymbol{Y}_{t}$. Therefore, we minimize:
\begin{equation}\label{eq:mi-ca}
\begin{split}
    \mathcal{L}_{CA,t}=&-\frac{1}{n^2}\sum_{i=1}^n\sum_{j=1}^n[\ \log q_{\theta_{CA}}(\boldsymbol{A}^i_{t}|\phi_C(\psi(\bar{\boldsymbol{H}}^i_{t-1}))) \\
    &-\log q_{\theta_{CA}}(\boldsymbol{A}^j_{t}|\phi_C(\psi(\bar{\boldsymbol{H}}^i_{t-1})))\ ],
\end{split}
\end{equation}
\begin{equation}\label{eq:mi-cy}
\begin{split}
    \mathcal{L}_{CY,t}=&-\frac{1}{n^2}\sum_{i=1}^n\sum_{j=1}^n[\ \log q_{\theta_{CY}}(\boldsymbol{Y}^i_{t}|\phi_C(\psi(\bar{\boldsymbol{H}}^i_{t-1}))) \\
    &-\log q_{\theta_{CY}}(\boldsymbol{Y}^j_{t}|\phi_C(\psi(\bar{\boldsymbol{H}}^i_{t-1})))\ ],
\end{split}
\end{equation}
(2) We require ${\boldsymbol{Z}}_{t-1} \perp\!\!\!\perp \bar{\boldsymbol{C}}_{t-1}$ in Theorem~\ref{theorem:IV}.
\begin{equation}\label{eq:mi-zc}
\begin{split}
    \mathcal{L}_{ZC,t}=\frac{1}{n^2}\sum_{i=1}^n\sum_{j=1}^n[\ &\log q_{\theta_{ZC}}(\boldsymbol{Z}_{t-1}^i|\bar{\boldsymbol{C}}_{t-1}^i) \\
    -&\log q_{\theta_{ZC}}(\boldsymbol{Z}_{t-1}^j|\bar{\boldsymbol{C}}_{t-1}^i)\ ],
\end{split}
\end{equation}
where $\boldsymbol{Z}_{t-1}$ and $\bar{\boldsymbol{C}}_{t-1}$ are approximated by $\phi_Z(\psi(\bar{\boldsymbol{H}}_{t-1}))$ and $\{\phi_C(\psi(\bar{\boldsymbol{H}}_{1})),\dots,\phi_C(\psi(\bar{\boldsymbol{H}}_{t-1}))\}$, respectively.

The overall loss function of these mutual information restrictions could be concluded as:
\begin{equation}\label{eq:mi-loss}
    \mathcal{L}_{MI,t}=\mathcal{L}_{ZA,t}+\mathcal{L}_{ZY,t}+\mathcal{L}_{CA,t}+\mathcal{L}_{CY,t}+\mathcal{L}_{ZC,t},
\end{equation}
where $\mathcal{L}_{MI}=\frac{1}{T}\sum_{t=1}^{T}\mathcal{L}_{MI}^t$ and $T$ is denoted as the length of historical time series. To optimize the parameters of variational distributions, i.e. $\{\theta_{ZA},\theta_{ZY},\theta_{CA},\theta_{CY},\theta_{ZC}\}$, we combine all the variational approximation loss as:
\begin{equation}\label{eq:lld}
\begin{split}
    \mathcal{L}_{LLD}=&\frac{1}{T}\sum_{t=1}^{T}(\mathcal{LLD}_{ZA,t}+\mathcal{LLD}_{ZY,t} \\ 
    + &\mathcal{LLD}_{CA,t}+\mathcal{LLD}_{CY,t}+\mathcal{LLD}_{ZC,t}),
\end{split}
\end{equation}
where the definitions of $\mathcal{LLD}_{ZY,t}$, $\mathcal{LLD}_{CA,t}$, $\mathcal{LLD}_{CY,t}$, and $\mathcal{LLD}_{ZC,t}$ are similar to that of $\mathcal{LLD}_{ZA,t}$ in Equation (\ref{eq:lld-za}), and we have omitted them 
due to space limitations. Using contrastive learning with upper bound for modeling representation, we automatically separate $\boldsymbol{Z}$ and $\boldsymbol{C}$ from $\boldsymbol{X}$ by Equation (\ref{eq:mi-zc}), while ensuring the validity of $\boldsymbol{Z}$ by Equation (\ref{eq:mi-za})-(\ref{eq:mi-zy}) and the confounding properties by Equation (\ref{eq:mi-ca})-(\ref{eq:mi-cy}). Although we have stated the static relationships between $\boldsymbol{Z}$, $\boldsymbol{C}$ and $\boldsymbol{X}$ in Section~\ref{sec:formulation}, our model is also applicable to the situation where the relationships are dynamic.

\begin{table*}[t]
\caption{Statistics of datasets.}
\label{tab:datasets}
\vskip 0.1in
\begin{center}\small
\begin{tabular}{ccccccccc}
\toprule
\textbf{Dataset} & \textbf{Type of} $\boldsymbol{A}$ & $\textbf{dim}_{\boldsymbol{A}}$ & $\textbf{dim}_{\boldsymbol{X}}$ &  $\textbf{dim}_{\boldsymbol{U}}$ & $\boldsymbol{T}$ & \textbf{\# Train} &  \textbf{\# Validation} &  \textbf{\# Test} \\
\midrule
Simulation & binary  & $1$ & $10$ & $3$ & $100$ & $10,000$ & $1,000$ & $1,000$ \\
Tumor growth & continuous & $2$ & $3$ & $3$ & $60$ & $10,000$ & $1,000$ & $21,741$ \\
Cryptocurrency & continuous & $1$ & $5$ & unknown & $5$ & $86$ & $29$ & $29$  \\
MIMIC-\uppercase\expandafter{\romannumeral 3} & binary & $2$ & $69$ & unknown & $60$ & $4,293$ & $920$ & $920$
\\
\bottomrule
\end{tabular}
\end{center}
\vskip -0.2in
\end{table*}

\begin{table*}[t]
\caption{Results of one-step-ahead outcome prediction on synthetic and real-world datasets.}
\label{tab:main-exp}
\vskip 0.1in
\begin{center}
\begin{small}
\begin{tabular}{ccccc}
\toprule
\textbf{Method} & \textbf{Simulation} & \textbf{Tumor} & \textbf{Cryptocurrency} & \textbf{MIMIC-\uppercase\expandafter{\romannumeral 3}} \\
\midrule
Time Series Deconfounder & $0.930\pm0.101$ & 
$0.753\pm0.042$ & $2.141\pm0.051$ & $1.071\pm0.039$ \\
Causal Transformer & $1.256\pm0.081$ & $0.716\pm0.003$ & $2.426\pm0.014$ & $1.229\pm0.026$ \\
ACTIN & $1.481\pm0.129$ & $1.041\pm0.002$ & $2.135\pm0.005$ & $1.162\pm0.043$ \\
ORL & $0.798\pm0.115$ & $0.347\pm0.010$ & $1.861\pm0.012$ & $1.253\pm0.024$ \\
Deep LTMLE & $0.531\pm0.077$ & $0.202\pm0.074$ & $1.553\pm0.061$ & $1.260\pm0.134$ \\ 
\midrule
\textbf{DSIV-CFR} & $\textbf{0.105}\boldsymbol{\pm}\textbf{0.017}$ & $\textbf{0.047}\boldsymbol{\pm}\textbf{0.003}$ & $\textbf{0.375}\boldsymbol{\pm}\textbf{0.039}$ & $\textbf{0.634}\boldsymbol{\pm}\textbf{0.019}$ \\
\bottomrule
\addlinespace[0.5em]
\multicolumn{3}{l}{\footnotesize * Lower = better (best in bold)} \\
\end{tabular}
\end{small}
\end{center}
\vskip -0.2in
\end{table*}

\subsection{GMM for Counterfactual Regression}\label{sec:agmm}

As discussed in Section~\ref{sec:formulation}, we establish a GMM framework to accurately estimate the potential outcome in the presence of unmeasured confounders. We also apply a transformer $h(\cdot)$ as the backbone of the estimator $\hat{\boldsymbol{Y}}_t=h(\boldsymbol{A}_t,\bar{\boldsymbol{A}}_{t-1},\bar{\boldsymbol{C}}_{t-1})$ mentioned in Equation (\ref{eq:2sls}). It is designed to predict the potential outcome $\boldsymbol{Y}_t$ from treatment $\boldsymbol{A}_t$ and a wealth of historical information $\{\bar{\boldsymbol{A}}_{t-1},\bar{\boldsymbol{C}}_{t-1}\}$, which can be seen as the confounders illustrated on the right of Figure~\ref{fig:causal-graph}. The basic prediction error is taken into account:
\begin{equation}\label{eq:mse-loss}
    \min_h \quad \mathcal{L}_{MSE}=\frac{1}{n\cdot T}\sum_{t=1}^{T}\sum_{i=1}^n\left(\hat{\boldsymbol{Y}}_t^i-\boldsymbol{Y}_t^i\right)^2,
\end{equation}
which could be viewed as a second-order moment constraint. To further implement the process described by Equation (\ref{eq:2sls}), we build a bridge function $f(\cdot)$ to obtain learnable weights $\boldsymbol{M}_{t-1}=f(\bar{\boldsymbol{A}}_{t-1},\bar{\boldsymbol{C}}_{t-1},\boldsymbol{Z}_{t-1})$. Objective is defined as:
\begin{equation}\label{eq:adv-loss}
    \min_h \max_f \quad \mathcal{L}_{adv}=\frac{1}{n\cdot T}\sum_{t=2}^{T+1}\sum_{i=1}^n\boldsymbol{M}_{t-1}^i\left(\hat{\boldsymbol{Y}}_t^i-\boldsymbol{Y}_t^i\right),
\end{equation}
where $\boldsymbol{M}_{t-1}^i$ is the weight of the $i$-th sample in $\boldsymbol{M}_{t-1}$.

The overall objective of our DSIV-CFR can be concluded as minimizing the following loss function:
\begin{equation}\label{eq:overall-loss}
    \min_h \max_f \mathcal{L}=\mathcal{L}_{MSE}+\alpha\cdot\mathcal{L}_{MI}+\beta\cdot\mathcal{L}_{adv},
\end{equation}
where $\{\alpha,\beta\}$ are two hyper-parameters to control the significance of different modules. We summarize the model optimization process as Algorithm~\ref{alg:workflow} in Appendix~\ref{app:codes}.

\section{Experiments}

\subsection{Baselines}

In Section~\ref{sec:related}, we discussed several works of causal inference on time series data. We applied these methods as baselines for the comparison of our DSIV-CFR, including Time Series Deconfounder~\cite{TSD}, surrogate-based approach~\cite{keith2021estimating}, Causal Transformer~\cite{CT}, ACTIN~\cite{ACTIN}, ORL~\cite{ORL}, and Deep LTMLE~\cite{DLTMLE}. More details are concluded in Appendix~\ref{app:exp-detail}.

\subsection{Datasets}

To comprehensively evaluate the performance of all models under various data conditions, we applied both synthetic data and real-world datasets. Statistics of these datasets are summarized in Table~\ref{tab:datasets}. The process of data generation for the fully-synthetic dataset, which we call \textbf{Simulation} dataset, is described in Appendix~\ref{app:exp-detail}.

\textbf{Tumor growth}~\cite{dataset-tumor}. Following previous works~\cite{CT,ACTIN}, we also applied a tumor growth simulator for data generation. We select the patient type and two static features as covariates. Treatments include chemotherapy and radiation therapy. Outcome of interest is the tumor volume. Considering the range of data, we normalized $\boldsymbol{X}$ and $\boldsymbol{Y}$.

\textbf{Cryptocurrency}\footnote{Cryptocurrency is available at \url{https://github.com/binance/binance-public-data}}. We also collected real-life time series data of Bitcoin and Ethereum from July $1$, $2023$ to October $29$, $2024$. Covariates include the opening price, lowest price, highest price, closing price, and trading volume for the day. Treatment refers to the positions held by individuals, while the outcome of interest is the return rate (\%). We manually divided the complete sequence into time segments of length $5$, setting the time gap to $5$ as well to prevent data leakage.

\textbf{MIMIC-\uppercase\expandafter{\romannumeral 3}}~\cite{dataset-mimic}. It is a comprehensive, publicly available database containing de-identified health data from patients admitted to critical care units at a large tertiary care hospital. Referring to data processing in the Causal Transformer~\cite{CT}, we obtain $69$ covariates on the vital signs of patients. Vasopressor and mechanical ventilation are two treatments taken into account. We also choose blood pressure as the outcome.

\subsection{Experimental Results}

\subsubsection{One-step-ahead prediction}

Our targeted estimand is the counterfactual outcome in the future time step, denoted as $\hat{\boldsymbol{Y}}_{t}(\boldsymbol{a}_t|\bar{\boldsymbol{H}}_{t-1})$. To evaluate the performance of these estimators, we use the MSE metric defined in Equation (\ref{eq:mse-loss}). To ensure the robustness and reliability of our experimental results, we independently repeated each experiment $5$ times, each with a different random seed. Finally, we report the mean and standard deviation (std) of MSE in the format of mean ± std. The experimental results of the four datasets are summarized in Table~\ref{tab:main-exp}, which comprehensively indicates the performance of the model in the diverse settings. It can be seen that our DSIV-CFR outperforms the state-of-the-art baselines on all datasets by a large margin.

The average running time of this experiment is detailed in Table~\ref{tab:time} (Appendix~\ref{app:time}). Although the introduction of mutual information loss $\mathcal{L}_{MI}$ and adversarial loss $\mathcal{L}_{adv}$ has led to a time increase, it is justified by the significant improvement in model performance. Generally, the overall training time remains within an acceptable range.

Since many baselines rely on the unconfoundedness assumption, i.e. do not consider the unmeasured confounders, we conducted an additional trial for fairness. We set all the confounding variables to be observable and retested these methods. Even in this case, where the baselines have access to more information than our approach, DSIV-CFR still achieves better performance.

\subsubsection{Hyper-parameter analysis}

\begin{figure}[t]
\begin{center}
\centerline{\includegraphics[width=0.5\textwidth]{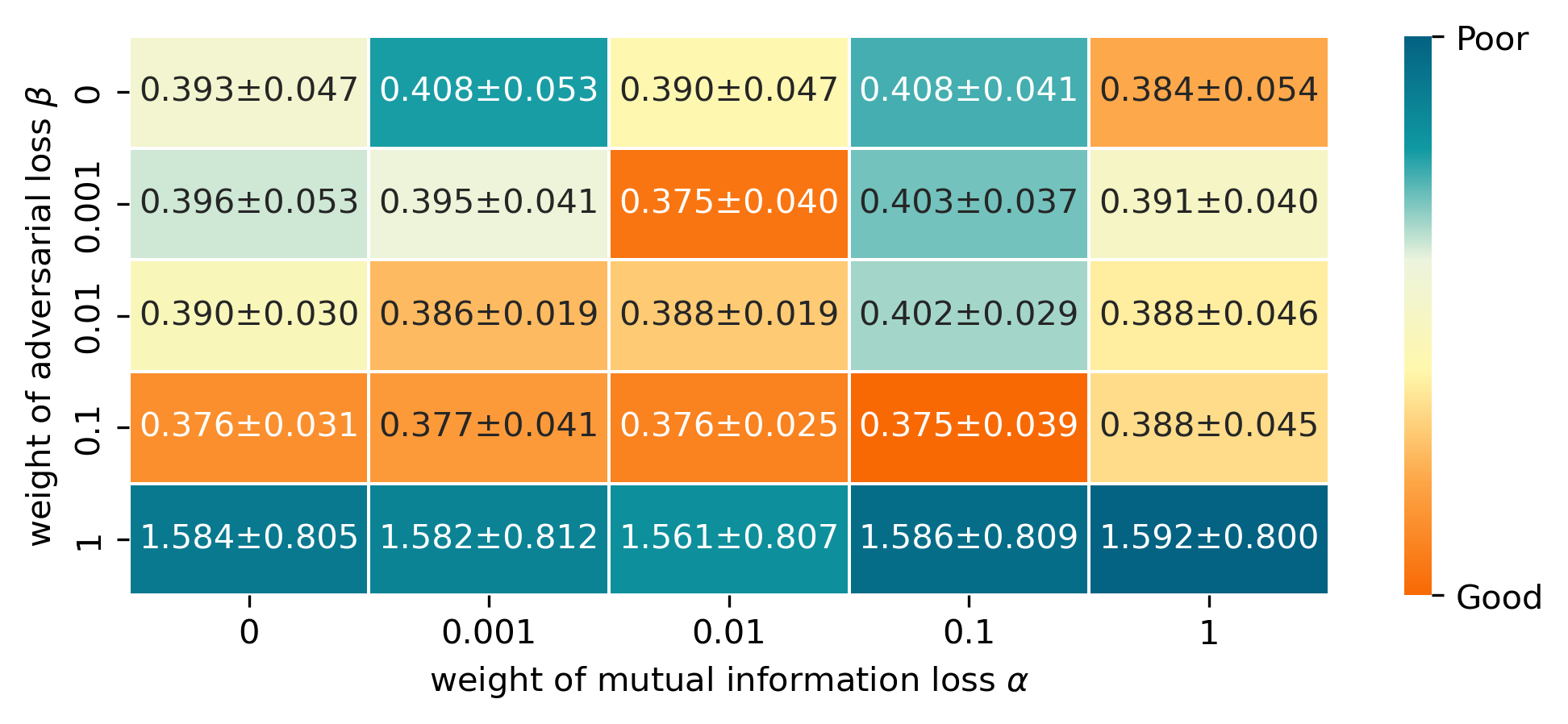}}
\vskip -0.1in
\caption{Results of hyper-parameter analysis on the Cryptocurrency dataset. Setting the $\alpha$ to $0$ is equivalent to removing $\mathcal{L}_{MI}$, and $\beta=0$ means $\mathcal{L}_{adv}$ is deleted. The heatmap on the right represents the changes in performance. If a cell's color is closer to a warm color (orange), it means that the model trained with the corresponding $\{\alpha,\beta\}$ combination has better performance. Conversely, if the cell's color is closer to a cool color (green), it indicates poorer performance.}
\label{fig:hyper}
\end{center}
\vskip -0.35in
\end{figure}

We explore the combination of $\alpha$ and $\beta$, which control the significance of mutual information constraints ($\mathcal{L}_{MI}$) for IV decomposition and adversarial function learning ($\mathcal{L}_{adv}$) in the GMM framework, respectively. When $\alpha$ or $\beta$ is set to $0$, it indicates that the corresponding module is ablated. Such evaluation is conducted with the Cryptocurrency dataset. According to Figure~\ref{fig:hyper}, the effectiveness of each module could be validated. If the two losses are moderately incorporated, model performance will be better than focusing solely on MSE. The best performance is achieved by setting $\alpha=0.1$ and $\beta=0.1$. However, if their weights are blindly increased, it may create a conflict with the primary objective, which is to minimize the prediction error.

\subsubsection{Sequential treatment decision making}

CT~\cite{CT} and ACTIN~\cite{ACTIN} can be extended to predict the future outcomes $\tau$ steps ahead, i.e. $\boldsymbol{Y}_{t+\tau}(\boldsymbol{a}_{t:t+\tau}|\bar{\boldsymbol{H}}_{t-1})$. The main idea is to additionally predict the next values of $\boldsymbol{X}$. In this way, in subsequent steps, the estimated values $\hat{\boldsymbol{X}}_{t:t+\tau-1}$ and $\hat{\boldsymbol{Y}}_{t:t+\tau-1}$ are used in place of the required observations to continue the cycle of prediction. Although our method could also be easily extended in the aforementioned manner, i.e. iteratively repeating the estimation of single time step to obtain the result of the ultimate moment, we attempted to directly predict the future outcomes of sequential treatments. We naturally consider applying this model to a downstream task, i.e. decision-making for the optimal sequential treatments in the next time period. Implementations of such methodology extension is clarified in Appendix~\ref{app:decision-making}. We conducted an experiment on multi-step-ahead decision-making and set $\tau$ to $5$. To obtain the oracle outcomes of all possible treatment sequences, we use the simulated dataset and selected the two baselines (CT and ACTIN) tailored for the similar task. $\boldsymbol{Y}_{t+\tau}(\boldsymbol{a}^*_{t:t+\tau})$ of decision making $\boldsymbol{a}^*_{t:t+\tau}$ compared to the oracle are illustrated in Figure~\ref{fig:decision}. Our method also performs well and achieves results that are close to the oracle.

\begin{figure}[t]
\begin{center}
\centerline{\includegraphics[width=0.5\textwidth]{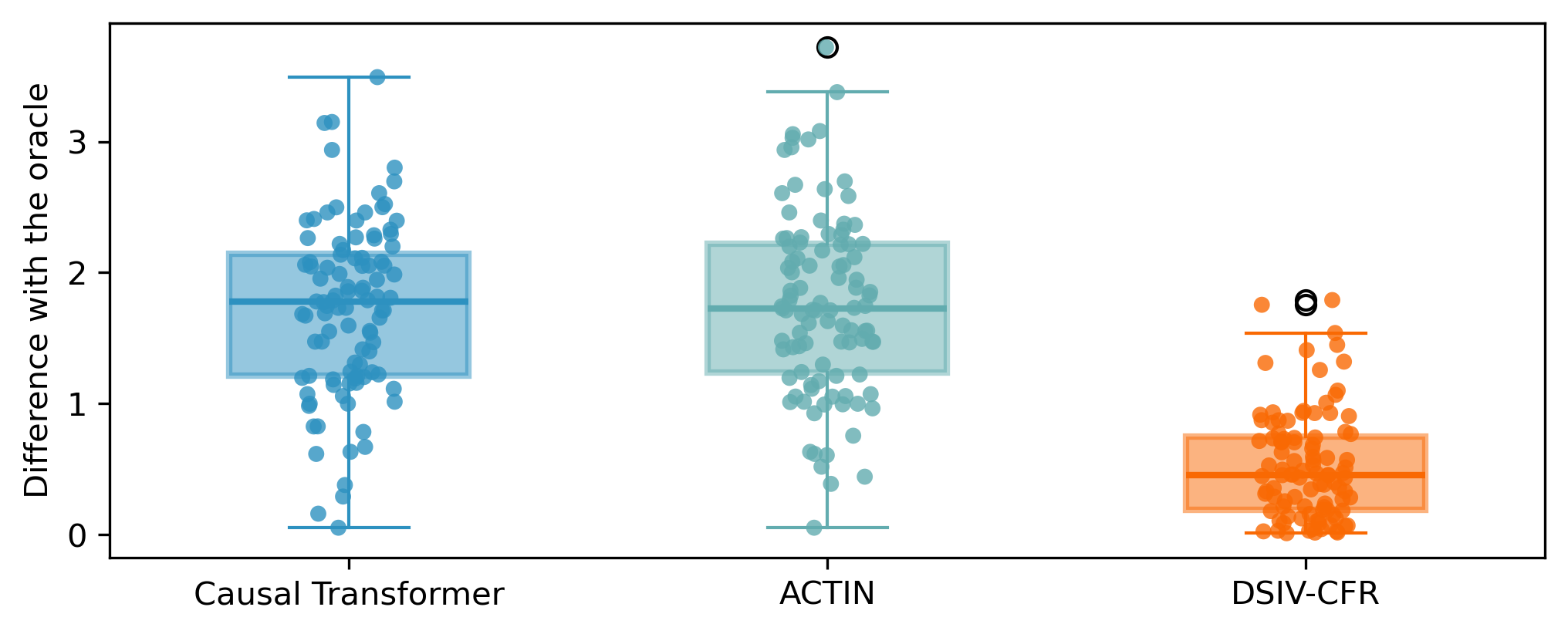}}
\vskip -0.1in
\caption{Results of decision making $5$ steps ahead. Values on the vertical axis represent the differences between oracle and the corresponding results of predicted optimal treatments (lower=better).}
\label{fig:decision}
\end{center}
\vskip -0.35in
\end{figure}

\section{Conclusion}

The presence of unmeasured confounders is a critical challenge of estimating the cumulative effects of time-varying treatments with time series data. We propose DSIV-CFR, a novel method that leverages negative controls and mutual information constraints to decompose IVs and confounders from the observations. DSIV-CFR accurately estimates the future potential outcome with unmeasured confounders through a modular design that combines transformers for sequential dependency modeling and GMM for counterfactual estimation with the learned IVs. This model can also be naturally extended for sequential treatment decision making, which holds significant potential for real-world scenarios.

\section*{Impact Statement}

This paper introduces a novel DSIV-CFR, which estimates the future potential outcome on time series data with unobserved confounders. This advancement in machine learning could improve decision making in various fields such as medicine, finance, and meteorology. For example, in personalized medicine, DSIV-CFR is able to predict the counterfactual outcome in the future of the treatment that may be taken for the next period, with the help of historical records. In this way, it allows doctors to provide better treatment strategies. The limitation of our method lies in the constraints mentioned in Assumption~\ref{ass:additive}. If $\mathbb{E}[\boldsymbol{U}|\boldsymbol{X}]\neq 0$, the estimate of potential outcome will be biased, but it still works to provide a consistent estimate of treatment effect.


\bibliography{example_paper}

\begin{thebibliography}{34}
\providecommand{\natexlab}[1]{#1}
\providecommand{\url}[1]{\texttt{#1}}
\expandafter\ifx\csname urlstyle\endcsname\relax
  \providecommand{\doi}[1]{doi: #1}\else
  \providecommand{\doi}{doi: \begingroup \urlstyle{rm}\Url}\fi

\bibitem[Ang et~al.(2006)Ang, Hodrick, Xing, and Zhang]{finance1}
Ang, A., Hodrick, R.~J., Xing, Y., and Zhang, X.
\newblock The cross-section of volatility and expected returns.
\newblock \emph{The journal of finance}, 61\penalty0 (1):\penalty0 259--299, 2006.

\bibitem[Ashenfelter(1978)]{did-orley}
Ashenfelter, O.
\newblock Estimating the effect of training programs on earnings.
\newblock \emph{The Review of Economics and Statistics}, 60\penalty0 (1):\penalty0 47--57, 1978.

\bibitem[Athey \& Imbens(2006)Athey and Imbens]{did-susan}
Athey, S. and Imbens, G.~W.
\newblock Identification and inference in nonlinear difference-in-differences models.
\newblock \emph{Econometrica}, 74\penalty0 (2):\penalty0 431--497, 2006.

\bibitem[Ba et~al.(2016)Ba, Kiros, and Hinton]{layer-norm}
Ba, L.~J., Kiros, J.~R., and Hinton, G.~E.
\newblock Layer normalization.
\newblock \emph{CoRR}, abs/1607.06450, 2016.
\newblock URL \url{http://arxiv.org/abs/1607.06450}.

\bibitem[Battocchi et~al.(2021)Battocchi, Dillon, Hei, Lewis, Oprescu, and Syrgkanis]{keith2021estimating}
Battocchi, K., Dillon, E., Hei, M., Lewis, G., Oprescu, M., and Syrgkanis, V.
\newblock Estimating the long-term effects of novel treatments.
\newblock In \emph{Advances in Neural Information Processing Systems 34: Annual Conference on Neural Information Processing Systems 2021, NeurIPS 2021, December 6-14, 2021, virtual}, pp.\  2925--2935, 2021.

\bibitem[Bica et~al.(2020)Bica, Alaa, and van~der Schaar]{TSD}
Bica, I., Alaa, A.~M., and van~der Schaar, M.
\newblock Time series deconfounder: Estimating treatment effects over time in the presence of hidden confounders.
\newblock In \emph{Proceedings of the 37th International Conference on Machine Learning, {ICML} 2020, 13-18 July 2020, Virtual Event}, volume 119 of \emph{Proceedings of Machine Learning Research}, pp.\  884--895. {PMLR}, 2020.

\bibitem[Cheng et~al.(2020)Cheng, Hao, Dai, Liu, Gan, and Carin]{CLUB}
Cheng, P., Hao, W., Dai, S., Liu, J., Gan, Z., and Carin, L.
\newblock {CLUB:} {A} contrastive log-ratio upper bound of mutual information.
\newblock In \emph{Proceedings of the 37th International Conference on Machine Learning, {ICML} 2020, 13-18 July 2020, Virtual Event}, volume 119 of \emph{Proceedings of Machine Learning Research}, pp.\  1779--1788. {PMLR}, 2020.

\bibitem[Elman(1990)]{RNN}
Elman, J.~L.
\newblock Finding structure in time.
\newblock \emph{Cogn. Sci.}, 14\penalty0 (2):\penalty0 179--211, 1990.

\bibitem[Feuerriegel et~al.(2024)Feuerriegel, Frauen, Melnychuk, Schweisthal, Hess, Curth, Bauer, Kilbertus, Kohane, and van~der Schaar]{medicine2}
Feuerriegel, S., Frauen, D., Melnychuk, V., Schweisthal, J., Hess, K., Curth, A., Bauer, S., Kilbertus, N., Kohane, I.~S., and van~der Schaar, M.
\newblock Causal machine learning for predicting treatment outcomes.
\newblock \emph{Nature Medicine}, 30\penalty0 (4):\penalty0 958--968, 2024.

\bibitem[G{\^a}rleanu \& Pedersen(2013)G{\^a}rleanu and Pedersen]{finance2}
G{\^a}rleanu, N. and Pedersen, L.~H.
\newblock Dynamic trading with predictable returns and transaction costs.
\newblock \emph{The Journal of Finance}, 68\penalty0 (6):\penalty0 2309--2340, 2013.

\bibitem[Geng. et~al.(2017)Geng., Paganetti, and Grassberger]{dataset-tumor}
Geng., C., Paganetti, H., and Grassberger, C.
\newblock Prediction of treatment response for combined chemo- and radiation therapy for non-small cell lung cancer patients using a bio-mathematical model.
\newblock \emph{Scientific Reports}, 7\penalty0 (1):\penalty0 13542, 2017.

\bibitem[Hall(2003)]{GMM}
Hall, A.~R.
\newblock Generalized method of moments.
\newblock \emph{A companion to theoretical econometrics}, pp.\  230--255, 2003.

\bibitem[Hartford et~al.(2021)Hartford, Veitch, Sridhar, and Leyton{-}Brown]{modelIV}
Hartford, J.~S., Veitch, V., Sridhar, D., and Leyton{-}Brown, K.
\newblock Valid causal inference with (some) invalid instruments.
\newblock In Meila, M. and Zhang, T. (eds.), \emph{Proceedings of the 38th International Conference on Machine Learning, {ICML} 2021, 18-24 July 2021, Virtual Event}, volume 139 of \emph{Proceedings of Machine Learning Research}, pp.\  4096--4106. {PMLR}, 2021.

\bibitem[Hizli et~al.(2023)Hizli, John, Juuti, Saarinen, Pietil{\"{a}}inen, and Marttinen]{policy2}
Hizli, C., John, S.~T., Juuti, A.~T., Saarinen, T.~T., Pietil{\"{a}}inen, K.~H., and Marttinen, P.
\newblock Causal modeling of policy interventions from treatment-outcome sequences.
\newblock In \emph{International Conference on Machine Learning, {ICML} 2023, 23-29 July 2023, Honolulu, Hawaii, {USA}}, volume 202 of \emph{Proceedings of Machine Learning Research}, pp.\  13050--13084. {PMLR}, 2023.

\bibitem[Hochreiter \& Schmidhuber(1997)Hochreiter and Schmidhuber]{LSTM}
Hochreiter, S. and Schmidhuber, J.
\newblock Long short-term memory.
\newblock \emph{Neural Comput.}, 9\penalty0 (8):\penalty0 1735--1780, 1997.

\bibitem[Huang \& Ning(2012)Huang and Ning]{medicine1}
Huang, X. and Ning, J.
\newblock Analysis of multi-stage treatments for recurrent diseases.
\newblock \emph{Statistics in medicine}, 31\penalty0 (24):\penalty0 2805--2821, 2012.

\bibitem[Jeunen et~al.(2020)Jeunen, Rohde, Vasile, and Bompaire]{policy1}
Jeunen, O., Rohde, D., Vasile, F., and Bompaire, M.
\newblock Joint policy-value learning for recommendation.
\newblock In \emph{{KDD} '20: The 26th {ACM} {SIGKDD} Conference on Knowledge Discovery and Data Mining, Virtual Event, CA, USA, August 23-27, 2020}, pp.\  1223--1233. {ACM}, 2020.

\bibitem[Johnson et~al.(2016)Johnson, Pollard, Shen, wei H.~Lehman, Feng, Ghassemi, Moody, Szolovits, Celi, and Mark]{dataset-mimic}
Johnson, A.~E., Pollard, T.~J., Shen, L., wei H.~Lehman, L., Feng, M., Ghassemi, M., Moody, B., Szolovits, P., Celi, L.~A., and Mark, R.~G.
\newblock Mimic-iii: a freely accessible critical care database.
\newblock \emph{Scientific Data}, 3\penalty0 (1):\penalty0 160035, 2016.

\bibitem[Kivva et~al.(2023)Kivva, Salehkaleybar, and Kiyavash]{yaroslav2023cross}
Kivva, Y., Salehkaleybar, S., and Kiyavash, N.
\newblock A cross-moment approach for causal effect estimation.
\newblock In \emph{Advances in Neural Information Processing Systems 36: Annual Conference on Neural Information Processing Systems 2023, NeurIPS 2023, New Orleans, LA, USA, December 10 - 16, 2023}, 2023.

\bibitem[Kuroki \& Pearl(2014)Kuroki and Pearl]{kuroki2014measurement}
Kuroki, M. and Pearl, J.
\newblock Measurement bias and effect restoration in causal inference.
\newblock \emph{Biometrika}, 101\penalty0 (2):\penalty0 423--437, 2014.

\bibitem[Melnychuk et~al.(2022)Melnychuk, Frauen, and Feuerriegel]{CT}
Melnychuk, V., Frauen, D., and Feuerriegel, S.
\newblock Causal transformer for estimating counterfactual outcomes.
\newblock In \emph{International Conference on Machine Learning, {ICML} 2022, 17-23 July 2022, Baltimore, Maryland, {USA}}, volume 162 of \emph{Proceedings of Machine Learning Research}, pp.\  15293--15329. {PMLR}, 2022.

\bibitem[Miao et~al.(2018)Miao, Geng, and Tchetgen]{miao2018identifying}
Miao, W., Geng, Z., and Tchetgen, E. J.~T.
\newblock Identifying causal effects with proxy variables of an unmeasured confounder.
\newblock \emph{Biometrika}, 105\penalty0 (4):\penalty0 pp. 987--993, 2018.

\bibitem[Miao et~al.(2024)Miao, Shi, Li, and Tchetgen]{miao2024bridge}
Miao, W., Shi, X., Li, Y., and Tchetgen, E.~T.
\newblock A confounding bridge approach for double negative control inference on causal effects, 2024.
\newblock URL \url{https://arxiv.org/abs/1808.04945}.

\bibitem[Moraffah et~al.(2021)Moraffah, Sheth, Karami, Bhattacharya, Wang, Tahir, Raglin, and Liu]{time-series-survey}
Moraffah, R., Sheth, P., Karami, M., Bhattacharya, A., Wang, Q., Tahir, A., Raglin, A., and Liu, H.
\newblock Causal inference for time series analysis: problems, methods and evaluation.
\newblock \emph{Knowl. Inf. Syst.}, 63\penalty0 (12):\penalty0 3041--3085, 2021.

\bibitem[Newey \& Powell(2003)Newey and Powell]{newey2003instrumental}
Newey, W.~K. and Powell, J.~L.
\newblock Instrumental variable estimation of nonparametric models.
\newblock \emph{Econometrica}, 71\penalty0 (5):\penalty0 1565--1578, 2003.

\bibitem[Robins \& Greenland(1986)Robins and Greenland]{robins1986role}
Robins, J.~M. and Greenland, S.
\newblock The role of model selection in causal inference from nonexperimental data.
\newblock \emph{American Journal of Epidemiology}, 123\penalty0 (3):\penalty0 392--402, 1986.

\bibitem[Shalit et~al.(2017)Shalit, Johansson, and Sontag]{CFR}
Shalit, U., Johansson, F.~D., and Sontag, D.~A.
\newblock Estimating individual treatment effect: generalization bounds and algorithms.
\newblock In \emph{Proceedings of the 34th International Conference on Machine Learning, {ICML} 2017, Sydney, NSW, Australia, 6-11 August 2017}, volume~70 of \emph{Proceedings of Machine Learning Research}, pp.\  3076--3085. {PMLR}, 2017.

\bibitem[Shirakawa et~al.(2024)Shirakawa, Li, Wu, Qiu, Li, Zhao, Iso, and van~der Laan]{DLTMLE}
Shirakawa, T., Li, Y., Wu, Y., Qiu, S., Li, Y., Zhao, M., Iso, H., and van~der Laan, M.~J.
\newblock Longitudinal targeted minimum loss-based estimation with temporal-difference heterogeneous transformer.
\newblock In \emph{Forty-first International Conference on Machine Learning, {ICML} 2024, Vienna, Austria, July 21-27, 2024}. OpenReview.net, 2024.

\bibitem[Srivastava et~al.(2014)Srivastava, Hinton, Krizhevsky, Sutskever, and Salakhutdinov]{dropout}
Srivastava, N., Hinton, G.~E., Krizhevsky, A., Sutskever, I., and Salakhutdinov, R.
\newblock Dropout: a simple way to prevent neural networks from overfitting.
\newblock \emph{J. Mach. Learn. Res.}, 15\penalty0 (1):\penalty0 1929--1958, 2014.

\bibitem[Tran et~al.(2024)Tran, Bibaut, and Kallus]{ORL}
Tran, A., Bibaut, A., and Kallus, N.
\newblock Inferring the long-term causal effects of long-term treatments from short-term experiments.
\newblock In \emph{Forty-first International Conference on Machine Learning, {ICML} 2024, Vienna, Austria, July 21-27, 2024}. OpenReview.net, 2024.

\bibitem[Vaswani et~al.(2017)Vaswani, Shazeer, Parmar, Uszkoreit, Jones, Gomez, Kaiser, and Polosukhin]{transformer}
Vaswani, A., Shazeer, N., Parmar, N., Uszkoreit, J., Jones, L., Gomez, A.~N., Kaiser, L., and Polosukhin, I.
\newblock Attention is all you need.
\newblock In \emph{Advances in Neural Information Processing Systems 30: Annual Conference on Neural Information Processing Systems 2017, December 4-9, 2017, Long Beach, CA, {USA}}, pp.\  5998--6008, 2017.

\bibitem[Wang et~al.(2024)Wang, Lyu, Yang, Zhan, and Chen]{ACTIN}
Wang, X., Lyu, S., Yang, L., Zhan, Y., and Chen, H.
\newblock A dual-module framework for counterfactual estimation over time.
\newblock In \emph{Forty-first International Conference on Machine Learning, {ICML} 2024, Vienna, Austria, July 21-27, 2024}. OpenReview.net, 2024.

\bibitem[Wang \& Blei(2019)Wang and Blei]{deconfounder}
Wang, Y. and Blei, D.~M.
\newblock The blessings of multiple causes.
\newblock \emph{Journal of the American Statistical Association}, 114\penalty0 (528):\penalty0 1574--1596, 2019.

\bibitem[Yuan et~al.(2022)Yuan, Wu, Kuang, Li, Wu, Wu, and Lin]{autoIV}
Yuan, J., Wu, A., Kuang, K., Li, B., Wu, R., Wu, F., and Lin, L.
\newblock Auto {IV:} counterfactual prediction via automatic instrumental variable decomposition.
\newblock \emph{{ACM} Trans. Knowl. Discov. Data}, 16\penalty0 (4):\penalty0 74:1--74:20, 2022.

\end{thebibliography}
\bibliographystyle{icml2025}

%%%%%%%%%%%%%%%%%%%%%%%%%%%%%%%%%%%%%%%%%%%%%%%%%%%%%%%%%%%%%%%%%%%%%%%%%%%%%%%
%%%%%%%%%%%%%%%%%%%%%%%%%%%%%%%%%%%%%%%%%%%%%%%%%%%%%%%%%%%%%%%%%%%%%%%%%%%%%%%
% APPENDIX
%%%%%%%%%%%%%%%%%%%%%%%%%%%%%%%%%%%%%%%%%%%%%%%%%%%%%%%%%%%%%%%%%%%%%%%%%%%%%%%
%%%%%%%%%%%%%%%%%%%%%%%%%%%%%%%%%%%%%%%%%%%%%%%%%%%%%%%%%%%%%%%%%%%%%%%%%%%%%%%
\newpage
\appendix
\onecolumn

\section{Pseudo-Code}\label{app:codes}

As stated in Section~\ref{sec:methodology}, we propose a novel DSIV-CFR method to accurately estimate the effect of sequential treatment in the presence of unmeasured confounders. It comprises two key modules. The first module utilizes transformers to model long sequential dependencies and captures the representations of instrumental variables (IVs) and confounders. The second module then employs the Generalized Method of Moments (GMM) to estimate counterfactual outcomes by effectively leveraging the learned representations. The detailed pseudo-code of DSIV-CFR is provided in Algorithm~\ref{alg:workflow}.

\begin{algorithm}[h]
   \caption{Decomposing Sequential Instrumental Variable framework for CounterFactual Regression}
   \label{alg:workflow}
\begin{algorithmic}
   \STATE {\bfseries Input:} dataset $\mathcal{D}=\{\bar{\boldsymbol{H}}_{t-1},\boldsymbol{A}_t\}_{t=1}^{T+1}$, transformers $\psi$ and $h$, representation heads $\phi_Z$ and $\phi_C$, bridge function $f$, hyper-parameter $\alpha$ and $\beta$, maximum iteration $K$, alternating training rounds $R$
   \STATE {\bfseries Output:} future outcomes $\hat{\boldsymbol{Y}}_{t=T+1}$
   \STATE Initialize parameters $\theta$ of $\{\psi,\phi_Z,\phi_C\}$, $\xi$ of $h$, $\zeta$ of $f$
   \FOR{$i\gets1$ \textbf{to} $K$}
        \FOR{$t\gets1$ \textbf{to} $T$}
            \STATE \textcolor{gray}{// Module 1: Long Sequential Modeling for Learning IV and Confounder Representations}
            \STATE Obtain basic embedding $\psi(\bar{\boldsymbol{H}}_{t-1})$
            \STATE Obtain instrumental variable representation $\boldsymbol{Z}_{t-1}\gets\phi_Z(\psi(\bar{\boldsymbol{H}}_{t-1}))$
            \STATE Obtain confounder representation $\boldsymbol{C}_{t-1}\gets\phi_C(\psi(\bar{\boldsymbol{H}}_{t-1}))$
            \STATE Calculate mutual information loss $\mathcal{L}_{MI}$ by Equation (\ref{eq:mi-loss})
            \STATE \textcolor{gray}{// Module 2: GMM for Counterfactual Regression}
            \STATE Predict potential outcome in the next time step $\hat{\boldsymbol{Y}}_t\gets h(\boldsymbol{A}_t,\bar{\boldsymbol{A}}_{t-1},\bar{\boldsymbol{C}}_{t-1})$
            \STATE Calculate MSE loss $\mathcal{L}_{MSE}$ by Equation (\ref{eq:mse-loss})
            \STATE Obtain learnable weights $\boldsymbol{M}_{t-1}\gets f(\bar{\boldsymbol{C}}_{t-1},\boldsymbol{Z}_{t-1})$
            \STATE Calculate adversarial loss $\mathcal{L}_{adv}$ by Equation (\ref{eq:adv-loss})
        \ENDFOR
        \STATE Calculate overall loss $\mathcal{L}\gets\mathcal{L}_{MSE}$ by Equation (\ref{eq:overall-loss})
        \STATE Update $\xi\gets\xi'$ after $\mathcal{L}$.backward( )
        \STATE \textcolor{gray}{// Alternating training for variational distribution approximation}
        \FOR{$j\gets1$ \textbf{to} $R$}
            \STATE Obtain $\bar{\boldsymbol{Z}}_{t-1}$ and $\bar{\boldsymbol{C}}_{t-1}$ by repeating Module 1
            \STATE Calculate log-likelihood loss $\mathcal{L}_{LLD}$ by Equation (\ref{eq:lld})
            \STATE Update $\theta\gets\theta'$ after $\mathcal{L}_{LLD}$.backward( )
        \ENDFOR
        \STATE \textcolor{gray}{// Alternating training for bridge function}
        \FOR{$j\gets1$ \textbf{to} $R$}
            \STATE Obtain $[\hat{Y}_t, \dots]$ and $[\boldsymbol{M}_{t-1},\dots]$ by repeating Module 2
            \STATE Calculate log-likelihood loss $\mathcal{L}_{adv}$ by Equation (\ref{eq:adv-loss})
            \STATE Update $\zeta\gets\zeta'$ after $\mathcal{L}_{adv}$.backward( )
        \ENDFOR
        \STATE Decompose historical observations $\bar{\boldsymbol{H}}_{T}$ into confounders $\bar{\boldsymbol{C}}_T$ and IV $\boldsymbol{Z}_T$ by networks $\{\psi,\phi_Z,\phi_C\}$
        \STATE Predict future outcomes $\hat{\boldsymbol{Y}}_{t=T+1}\gets h(\boldsymbol{A}_{T+1},\bar{\boldsymbol{A}}_T,\bar{\boldsymbol{C}}_T)$
   \ENDFOR
   \STATE \textbf{return} $\hat{\boldsymbol{Y}}_{t=T+1}$
\end{algorithmic}
\end{algorithm}

\newpage
\section{Implementation details of experiments}\label{app:exp-detail}

\textbf{Baselines.} We have concluded some details of the baselines applied in this paper, which are shown in Table~\ref{tab:baseline}. Among them, Time Series Deconfounder~\cite{TSD}, abbreviated as TSD, considers the scenario with unobserved confounders, while other methods all rely on the unconfoundedness assumption. We also clarify their backbones for modeling the time series data. Implementations of them are all available at the given links. Although there is also a surrogate-based method~\cite{keith2021estimating}, we have not included it in the baselines for comparison temporarily, as we have not found its open-source code.

\begin{table}[h]
\caption{Details of baselines.}
\label{tab:baseline}
\vskip 0.1in
\begin{center}
\begin{small}
\begin{tabularx}{\textwidth}{c|cc|>{\centering\arraybackslash}m{11cm}} 
\toprule
\textbf{Method} & \textbf{Consider} $\boldsymbol{U}$ & \textbf{Backbone} & \textbf{Open-source} \\
\midrule
TSD & $\checkmark$ & RNN & \url{https://github.com/vanderschaarlab/mlforhealthlabpub/tree/main/alg/time_series_deconfounder} \\
\midrule
CT & $\times$ & Transformer & \url{https://github.com/Valentyn1997/CausalTransformer} \\
ACTIN & $\times$ & LSTM & \url{https://github.com/wangxin/ACTIN} \\
ORL & $\times$ & RL & \url{https://github.com/allentran/long-term-ate-orl} \\
DLTMLE & $\times$ & Transformer & \url{https://github.com/shirakawatoru/dltmle-icml-2024} \\ 
\bottomrule
\end{tabularx}
\end{small}
\end{center}
\end{table}

\textbf{Data generation for one-step ahead prediction.} For each moment, we generate $3$-dimensional IVs $\boldsymbol{Z}_t$ and $7$-dimensional confounders $\boldsymbol{C}_t$, each dimension of which follows a uniform distribution $\mathcal{U}(0,1)$. There are also unobserved confounders $\boldsymbol{U}_t$ of $3$ dimensions, and each dimension is randomly sampled from $\mathcal{N}(0,1)$. Specifically,
\begin{equation}
    \left\{
\begin{aligned}
  &\ \boldsymbol{Z}_{t+1}\gets 0.4\cdot\boldsymbol{Z}_t+0.6\cdot\boldsymbol{Z}_{t+1}+0.3\cdot\sin t \\
  &\ \boldsymbol{C}_{t+1}\gets 0.3\cdot\boldsymbol{C}_t+0.7\cdot\boldsymbol{C}_{t+1}+0.2\cdot\sin t \\
  &\ \boldsymbol{U}_{t+1}\gets\boldsymbol{U}_{t+1}-0.1\cdot\cos t
\end{aligned}
\right.
\end{equation}
As for the training and validation data, their treatments $\boldsymbol{A}_{t+1}$ of each time step is simulated by the following process.
\begin{equation}\label{eq:simulation-a}
\begin{split}
    &\text{logit}=\sum_{i=1}^d\left(coef_{a,i}\cdot\boldsymbol{V}_{t,i}-\cos{\boldsymbol{V}_{t,i}^2}\right)-0.5\cdot\boldsymbol{A}_t+0.2\cdot\boldsymbol{Y}_t -0.1\cdot\sin t, \quad \boldsymbol{V}_t=\{\boldsymbol{Z}_t, \boldsymbol{C}_t, \boldsymbol{U}_t\} \\
    &\mathbb{P}(\boldsymbol{A}_{t+1})=\frac{1}{1+\exp(-\text{logit})}, \quad \boldsymbol{A}_{t+1}=\left\{
\begin{aligned}
  &\ 0,\ \text{if}\ \ \mathbb{P}(\boldsymbol{A}_{t+1})< 0.5 \\
  &\ 1,\ \text{if}\ \ \mathbb{P}(\boldsymbol{A}_{t+1})\geq 0.5
\end{aligned}
\right.
\end{split}
\end{equation}
The coefficients $coef_a$ are randomly generated. We initialized that $\boldsymbol{A}_0=[0,\dots,0]$ and $\boldsymbol{Y}_0=[0,\dots,0]$. To evaluate the counterfactual prediction capabilities of each model, sequences $\boldsymbol{A}$ are randomly generated in the test set. In all sets, the generation of $\boldsymbol{Y}_{t+1}$ can be described as :
\begin{equation}\label{eq:simulation-y}
    \boldsymbol{Y}_{t+1}=\sum_{i=1}^d\left(coef_{y,i}\cdot\boldsymbol{V}_{t,i}'\right)-0.2\cdot\sin\boldsymbol{A}_{t+1}+0.5\cdot\sin\frac{t}{5},\quad \boldsymbol{V}_t'=\{\bar{\boldsymbol{C}}_t,\boldsymbol{U}_t,\boldsymbol{U}_{t-1}\}
\end{equation}

\newpage
\section{Sequential treatment decision making}\label{app:decision-making}

\textbf{Methodology extension.} Our method can be naturally extended to make decisions about determining the optimal treatment plan $\tau$ steps ahead. Specifically, the treatment $\boldsymbol{a}_{t:t+\tau-1}$ taken into consideration is a sequence of length $\tau$, where each dimension can independently take one of $d$ possible treatments. For example, in our setting, we set the time window $\tau$ to $5$ and the treatment to be binary ($d=2$). We first traverse all of the $2^5=32$ possible alternatives, infer the corresponding outcomes, and select the optimal one. For each subsequent moment, we still use the transformer $\psi(\cdot)$ in Section~\ref{sec:IV} to learn a basic representation of the observations $\{\bar{\boldsymbol{X}}_{t+3},\bar{\boldsymbol{A}}_{t+3},\bar{\boldsymbol{Y}}_{t-1}\}$, followed by $\phi_Z$ and $\phi_C$ to obtain the IVs $\{\boldsymbol{Z}_{t-1},\boldsymbol{Z}_{t},\boldsymbol{Z}_{t+1},\boldsymbol{Z}_{t+2},\boldsymbol{Z}_{t+3}\}$ and confounders $\{\boldsymbol{C}_{t-1},\boldsymbol{C}_{t},\boldsymbol{C}_{t+1},\boldsymbol{C}_{t+2},\boldsymbol{C}_{t+3}\}$. It is important to note that, since IVs are exogenous variables, the covariates for decomposition, i.e. $\{\boldsymbol{X}_{t-1}, \boldsymbol{X}_{t},\boldsymbol{X}_{t+1},\boldsymbol{X}_{t+2},\boldsymbol{X}_{t+3}\}$, must be observable in our setting. If we were to predict $\boldsymbol{X}_{t-1:t+3}$ on our own, as in the case of CT~\cite{CT}, it would not include the IVs. We then modify the counterfactual regression module mentioned in Section~\ref{sec:agmm}, where the second transformer $h(\cdot)$ to directly predict the $5$-step-ahead outcome by $\boldsymbol{Y}_{t+4}=h(\boldsymbol{A}_{t:t+4},\bar{\boldsymbol{A}}_{t-1},\bar{\boldsymbol{C}}_{t+3})$. As for the adversarial loss, we rebuild the bridge function $f(\cdot)$ to learn the weights $\boldsymbol{M}_{t+3}=f(\bar{\boldsymbol{A}}_{t+3}, \bar{\boldsymbol{C}}_{t+3}, \boldsymbol{Z}_{t-1:t+3})$. The other parts of our DSIV-CFR remain unchanged, as described in the previous sections.

\textbf{Data generation for five-step ahead decision making.} The dimension of $\boldsymbol{Z}$, $\boldsymbol{C}$, and $\boldsymbol{U}$ are $3$, $12$, and $5$, respectively. Each dimension of $\boldsymbol{Z}$ and $\boldsymbol{C}$ follows a uniform distribution $\mathcal{U}(0,3)$. As for the training and validation sets, we respectively generate $2,000$ and $200$ samples roughly following the data generation process described in Appendix~\ref{app:exp-detail}, except that $\boldsymbol{Y}$ will be affected by the interventions in the prior $5$ time steps, which can be expressed by the following formulation:
\begin{equation}
\boldsymbol{Y}_{t+1}=0.2\sum_{i=1}^d\left(coef_{y,i}\cdot\boldsymbol{V}_{t,i}'\right)-0.5\sum_{j=0}^4coef_{seq,j}\boldsymbol{A}_{t+1-j}+\sin t,\quad \boldsymbol{V}_t'=\{\bar{\boldsymbol{C}}_t,\boldsymbol{U}_t,\boldsymbol{U}_{t-1}\},
\end{equation}
where coefficients $coef_{seq}$ are randomly generated. These two sets include the information of $\{\bar{\boldsymbol{H}}_{29},\boldsymbol{Y}_{30}\}$. To generate the test set for evaluation of decision making, we first generate the historical information of $25$ time steps following the aforementioned steps. Afterwards, for each treatment plan $b$ in the $32$ alternatives $\mathcal{B}$, we directly apply the assigned interventions $\{\boldsymbol{A}_{26}^b,\dots,\boldsymbol{A}_{30}^b\}$ instead of inferring them with Equation (\ref{eq:simulation-a}), to iteratively calculate $\{\boldsymbol{Y}_{26}^b,\dots,\boldsymbol{Y}_{30}^b\}$ by Equation (\ref{eq:simulation-y}). The oracle is recorded as $\max_{b\in\mathcal{B}}\boldsymbol{Y}_{30}^b$, meaning the maximum outcome at timestamp $T=30$ after the optimal treatment plan $b^*$ is applied. This set includes $100$ records covering information of $\{\bar{\boldsymbol{X}}_{29},\bar{\boldsymbol{A}}_{25},\bar{\boldsymbol{Y}}_{25},\boldsymbol{Y}^*_{30}\}$. The evaluation results of model's early decision making ability compared with the oracle are visualized in Figure~\ref{fig:decision-data}.

\textbf{Real-world applications.} Our model is capable of forecasting future outcomes over multiple steps ahead and identifying the optimal treatments through exhaustive computation. This capability endows it with broad applicability and significant potential for generalization. For example, in the field of autonomous driving, our model can predict the future states of the road environment, such as the trajectories of vehicles and pedestrians, as well as changes in traffic signals. This allows it to plan the optimal driving route in advance. In the medical field, the model can predict the progression of a patient's condition and calculate the best treatment plan ahead of time. In the economic and financial sectors, the model can be used to forecast market trends and changes in economic indicators, thereby assisting investors in devising optimal investment strategies.

\begin{figure}[h]
    \begin{minipage}[h]{0.45\textwidth}
    \centering
    \includegraphics[width=\linewidth]{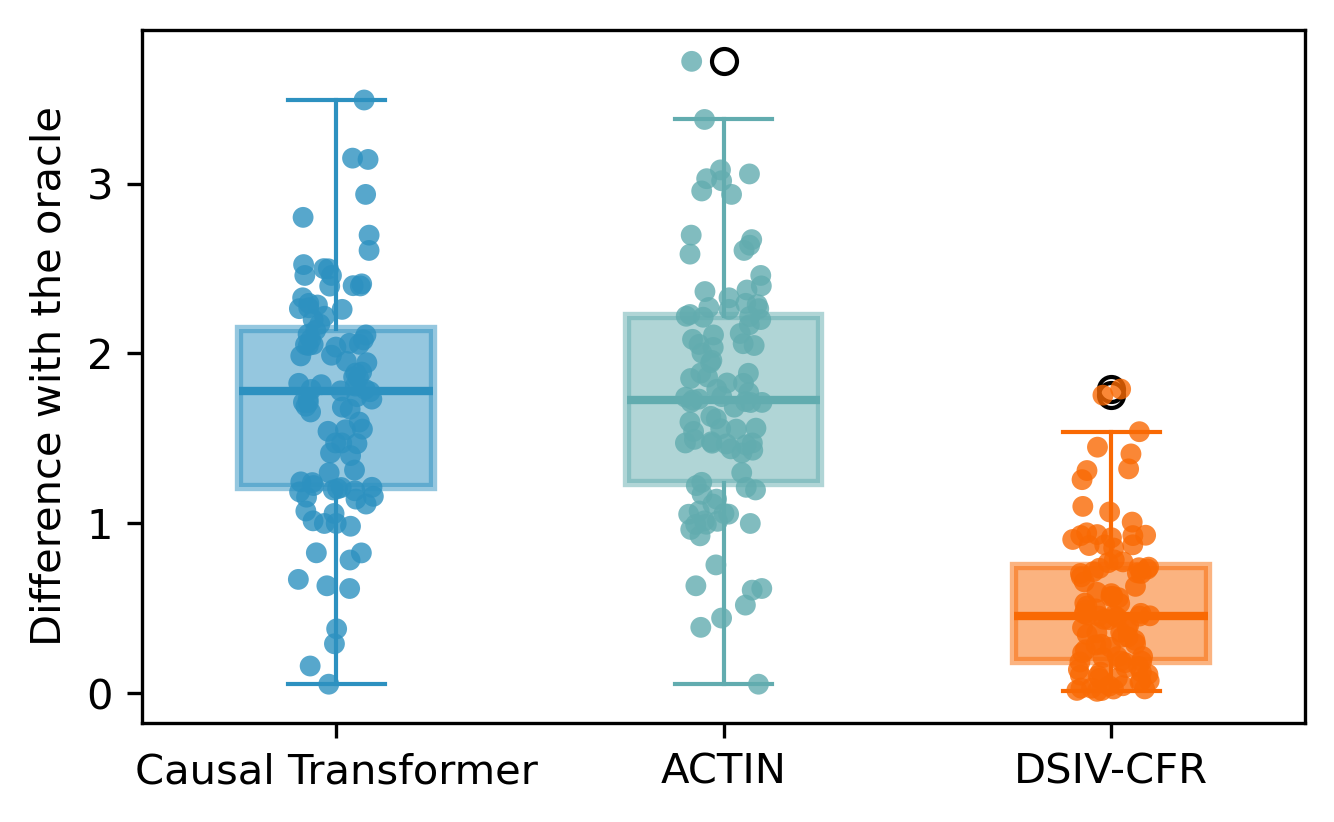}
    \end{minipage}
    \hfill 
    \begin{minipage}[t]{0.48\textwidth}
        \centering
        \begin{tabular}{ccccc}       
            \multicolumn{5}{c}{\small \textbf{Statistics of policy making results compared to oracle}} \\
            \addlinespace[0.5em]
            \toprule
            & \textbf{Min} & \textbf{Max} & \textbf{Avg} & \textbf{Std} \\
            \midrule
            CT & $0.054$ & $3.491$ & $1.722$ & $0.657$ \\
            ACTIN & $0.054$ & $3.719$ & $1.759$ & $0.692$ \\
            \midrule
            \textbf{DSIV-CFR} & $\boldsymbol{0.012}$ & $\boldsymbol{1.791}$ & $\boldsymbol{0.529}$ & $\boldsymbol{0.407}$ \\
            \bottomrule
            \addlinespace[0.5em]
\multicolumn{5}{l}{\footnotesize * Lower = better (best in bold)} \\
        \end{tabular}
    \end{minipage}
    \caption{Results of decision making $5$ steps ahead. It is a detailed version of Figure~\ref{fig:decision}. Visualization is given on the left, where values on the vertical axis represent the differences between oracle and the corresponding results of predicted optimal treatments (lower=better). Detailed statistics are reported on the right.}
     \label{fig:decision-data}
\end{figure}

\newpage
\section{Time complexity}\label{app:time}

When conducting the experiment of one-step-ahead outcome prediction, we also record the training time of each method included in evaluation. Results are shown in Table~\ref{tab:time}. The larger sample size, longer time series, and higher feature dimensionality could require more time for training. In addition, $\mathcal{L}_{MI}$ and $\mathcal{L}_{adv}$ consume much longer time than $\mathcal{L}_{MSE}$, but they play an important role in improving the model performance.

\begin{table*}[h]
\caption{Running time (minutes) of one-step-ahead prediction.}
\label{tab:time}
\vskip 0.1in
\begin{center}
\begin{small}
\begin{tabular}{ccccc}
\toprule
\textbf{Method} & \textbf{Simulation} & \textbf{Tumor} & \textbf{Cryptocurrency} & \textbf{MIMIC-\uppercase\expandafter{\romannumeral 3}} \\
\midrule
Time Series Deconfounder & $16.55$ & $3.200$ & $0.200$ & $2.450$ \\
Causal Transformer & $15.27$ & $2.100$ & $0.417$ & $1.970$ \\
ACTIN & $17.05$ & $1.633$ & $0.550$ & $0.267$ \\
ORL & $13.78$ & $3.100$ & $0.333$ & $2.800$ \\
Deep LTMLE & $9.267$ & $3.350$ & $0.183$ & $2.217$ \\ 
\midrule
\textbf{DSIV-CFR} & $27.87$ & $25.75$ & $1.317$ & $10.30$ \\
\bottomrule
\end{tabular}
\end{small}
\end{center}
\vskip 0.2in
\end{table*}

\section{Experimental results without unmeasured confounders}

Many relevant methods do not take into account the impact of unobserved confounders. However, in the evaluation data used in Table~\ref{tab:main-exp}, there exists the influence of $\boldsymbol{U}$ on $\boldsymbol{Y}$. Therefore, we set $U$ as observable, satisfying the unconfoundedness assumption, and re-evaluated the performance of the baselines. Results are reported in Table~\ref{tab:baseline-up}. If $U$ contains a significant amount of important information, it can be beneficial for improving the performance of outcome prediction. However, if it contains more noise, it may instead lead to a decline in performance.

\begin{table}[h]
\caption{Results of baselines for one-step-ahead prediction without unmeasured confounders.}
\label{tab:baseline-up}
\vskip 0.1in
\begin{center}
\begin{small}
\begin{tabular}{cccc}
\toprule
\textbf{Method} & \textbf{Input} & \textbf{Simulation} & \textbf{Tumor} \\
\midrule
Time Series Deconfounder & $\boldsymbol{X}$ & $0.930\pm0.101$ & $0.753\pm0.042$\\
\midrule
\multirow{2}{*}{Causal Transformer} & $\boldsymbol{X}$ & $1.256\pm0.081$ & $0.716\pm0.003$ \\
& $\{\boldsymbol{X},\boldsymbol{U}\}$ & $0.804\pm0.007$ & $0.703\pm0.013$  \\
\midrule
\multirow{2}{*}{ACTIN} & $\boldsymbol{X}$ & $1.481\pm0.129$ & $1.041\pm0.002$ \\
& $\{\boldsymbol{X},\boldsymbol{U}\}$ & $1.355\pm0.152$ & $1.049\pm0.003$  \\
\midrule
\multirow{2}{*}{ORL} & $\boldsymbol{X}$ & $0.798\pm0.115$ & $0.347\pm0.010$ \\
& $\{\boldsymbol{X},\boldsymbol{U}\}$ & $0.788\pm0.109$ & $0.351\pm0.007$   \\
\midrule
\multirow{2}{*}{Deep LTMLE} & $\boldsymbol{X}$ & $0.531\pm0.077$ & $0.202\pm0.074$ \\
& $\{\boldsymbol{X},\boldsymbol{U}\}$ & $0.467\pm0.041$ & $0.179\pm0.073$ \\ 
\midrule
\textbf{DSIV-CFR} & $\boldsymbol{X}$ & $\textbf{0.105}\boldsymbol{\pm}\textbf{0.017}$ & $\textbf{0.047}\boldsymbol{\pm}\textbf{0.003}$ \\
\bottomrule
\addlinespace[0.8em]
\multicolumn{3}{l}{\footnotesize * Lower = better (best in bold)} \\
\end{tabular}
\end{small}
\end{center}
\end{table}

%%%%%%%%%%%%%%%%%%%%%%%%%%%%%%%%%%%%%%%%%%%%%%%%%%%%%%%%%%%%%%%%%%%%%%%%%%%%%%%
%%%%%%%%%%%%%%%%%%%%%%%%%%%%%%%%%%%%%%%%%%%%%%%%%%%%%%%%%%%%%%%%%%%%%%%%%%%%%%%

\end{document}